\documentclass[letterpaper,journal]{IEEEtran}
\usepackage[T1,OT1]{fontenc}
\usepackage{amsmath,amsfonts}
\usepackage{algorithmic}
\usepackage{algorithm}
\usepackage{array}
\usepackage[caption=false,font=normalsize,labelfont=sf,textfont=sf]{subfig}
\usepackage{textcomp}
\usepackage{stfloats}
\usepackage{cuted}
\usepackage{placeins}
\usepackage{xpatch}
\usepackage{url}
\usepackage{verbatim}
\usepackage{graphicx}
\usepackage{cite}
\usepackage{microtype}
\usepackage{graphicx}
\usepackage{booktabs} 
\usepackage{color}
\usepackage{xr-hyper}
\usepackage[hidelinks]{hyperref}
\usepackage{multirow} %
\usepackage{amsmath}
\usepackage{amssymb}
\usepackage{mathtools}
\usepackage{amsthm}

\usepackage[capitalize,noabbrev]{cleveref}

\theoremstyle{definition}

\theoremstyle{plain}
\newtheorem{theorem}{Theorem}[section]
\newtheorem{proposition}[theorem]{Proposition}

\theoremstyle{definition}
\newtheorem{definition}[theorem]{Definition}

\theoremstyle{remark}

\usepackage[table,xcdraw]{xcolor}
\definecolor{lightgray}{gray}{0.92}
\newcommand{\xmark}{\ding{55}}

\usepackage{amsmath}
\usepackage{amssymb}
\usepackage{mathtools}
\usepackage{amsthm}

\usepackage{pifont}       
\usepackage{bbding}       %
\usepackage{pifont}    %
\usepackage[capitalize,noabbrev]{cleveref}
\renewcommand{\thesection}{\arabic{section}}

\theoremstyle{plain}

\theoremstyle{definition}

\theoremstyle{remark}

\begin{document}

\title{Counterfactual Feature Regularization for Expansion-Based Class-Incremental Learning}

\author{
Zhen~Zhang,
Jielei~Chu,~\IEEEmembership{Senior~Member,~IEEE,}
Jiangtao~Hu,
Bin~Liu,
Jie~Wang,
Ya~Liu,
Fengmao~Lv,~\IEEEmembership{Member,~IEEE,}
Tianrui~Li,~\IEEEmembership{Senior~Member,~IEEE}
\thanks{Zhen Zhang is with the School of Computing and Artificial Intelligence, Southwest Jiaotong University, Chengdu 611756, China. (e-mail: zhenzhang@my.swjtu.edu.cn). Jielei Chu, Fengmao Lv, Tianrui Li, Bin Liu and Jie Wang are with the School of Computing and Artificial Intelligence, Southwest Jiaotong University, Chengdu 611756, China. Jielei Chu and Tianrui Li are also with the Engineering Research Center of Sustainable Urban Intelligent Transportation, Ministry of Education, Chengdu, Sichuan 611756, China, the National Engineering Laboratory of Integrated Transportation Big Data Application Technology, Southwest Jiaotong University, Chengdu 611756, China. (e-mail: {jieleichu, fengmaolv, trli}@swjtu.edu.cn).
Jiangtao Hu, is with Sichuan International Travel Healthcare Center, Chengdu Customs, Chengdu, 610041, Sichuan, China  (e-mail: hjt9171@163.com)
Ya Liu is with Chengdu University of Traditional Chinese Medicine, Chengdu 610072, China. (e-mail: liuyaya918@163.com)
\vspace{1em} 

Jielei Chu is the corresponding author.}
}

\maketitle

\begin{abstract}
Expansion-based class-incremental learning (CIL) mitigates
catastrophic forgetting by freezing previously learned features
and promoting diversity between historical and newly expanded
task-specific representations. However, these strategies do
not necessarily ensure discriminative feature expansion and
may instead lead to feature collapse. Within each task,
empirical risk minimization (ERM) may favor spurious shortcuts
over class-relevant semantics, resulting in incomplete
representations. This shortcut dependence can be further
amplified across tasks, where diversity-driven feature
learning may suppress informative shared semantics between
visually similar classes. Consequently, newly learned
features may concentrate on a limited set of cues, leading
to feature collapse and cross-task semantic confusion.
To address this issue, we propose CPNS, a regularization
method based on the Probability of Necessity and Sufficiency
(PNS) for expansion-based CIL. CPNS evaluates the sufficiency and necessity of newly learned features for current-task and joint classification. To estimate these quantities, we construct intra-task and inter-task counterfactual features using perturbations guided by competing-class KNN neighborhoods. The resulting regularization encourages correct factual predictions and penalizes counterfactual replacements that retain the true-label prediction, promoting intra-task class discrimination and inter-task feature separability. The proposed regularization
can be readily integrated into existing expansion-based
CIL methods. Our theoretical analysis establishes the identifiability of CPNS, provides guarantees for intra-task class discrimination and inter-task feature separability, and derives a generalization bound for CPNS risk. Extensive experiments on three standard
and four fine-grained classification datasets demonstrate
its effectiveness.
\end{abstract}


\section{Introduction}

Class-incremental learning (CIL) aims to learn
new classes sequentially while recognizing all
previously observed classes without task identities
at inference time~\cite{rebuffi2017icarl,zhou2024class}.
Expansion-based methods mitigate catastrophic
forgetting by freezing previously learned feature
extractors and adding new branches for subsequent
tasks~\cite{yan2021dynamically}.
This separation preserves established representations
while providing additional capacity for new visual
concepts. It offers a direct way to balance knowledge
retention and adaptation as the class vocabulary grows.
However, preserving historical features does not
ensure that they capture all attributes needed
for cross-task discrimination.
Under empirical risk minimization
(ERM)~\cite{vapnik1999overview}, models may favor
a small subset of predictive cues while leaving
other informative attributes weakly
encoded~\cite{pezeshki2021gradient,pmlr-v235-qiu24e}.
When these cues depend on training-specific
correlations, the learned representations may
rely on shortcuts that fail under distribution
shifts~\cite{geirhos2020shortcut,pmlr-v238-yang24c,ICLR2026_05f7fb7b}.
Once frozen, historical feature extractors
cannot recover these missing attributes
through further updates.
Consequently, their representations may lack
the fine-grained information needed to distinguish
visually similar classes introduced in later tasks.

\begin{figure}[t]
\centering
\includegraphics[width=0.48\textwidth]{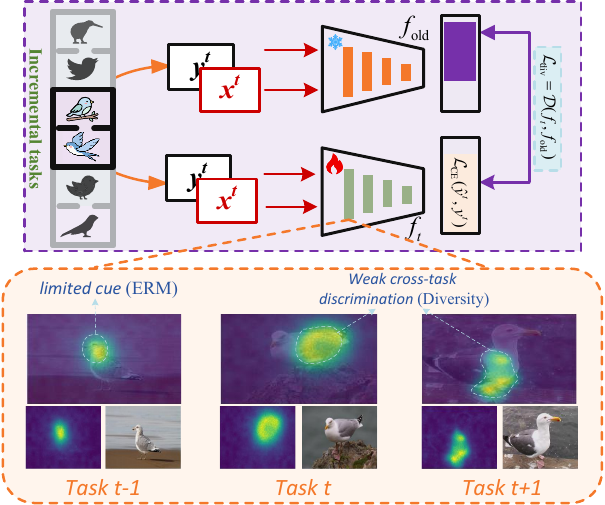}
\caption{Illustration of feature collision in
expansion-based CIL. ERM favors limited
predictive cues within tasks, while
diversity-driven expansion weakens
discrimination between visually similar
classes across tasks.
}
\label{motivation}
\end{figure}

To promote feature diversity, some expansion-based
methods encourage new branches to learn features
distinct from historical representations.
However, feature diversity does not necessarily
improve discrimination between similar classes
across tasks.
When these classes share informative attributes,
new branches may favor alternative predictive cues
rather than capture their fine-grained differences.
Combined with ERM's preference for a limited set
of predictive cues, this tendency can leave
important cross-task distinctions weakly encoded.
As a result, newly learned features may concentrate
on a narrow subset of cues, leading to feature
collapse and increased confusion between
visually similar classes.

\begin{figure*}[t]
\centering
\includegraphics[width=\textwidth]{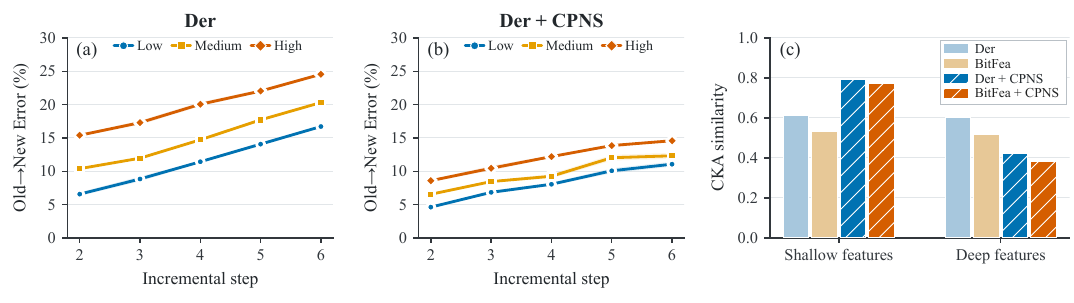}
\caption{(a,b) Old-to-new misclassification rates on CUB200
for DER and DER with CPNS, respectively, grouped by semantic
overlap. Greater overlap is associated with higher error rates.
(c) CKA similarity of shallow and deep features on ImageNet-1K.
CPNS increases shallow-feature similarity and decreases
deep-feature similarity relative to the corresponding baselines.}
\label{motivation_data}
\end{figure*}

Fig.~\ref{motivation} illustrates this phenomenon
in class-incremental experiments with DER
on the fine-grained CUB-200 dataset.
We use Grad-CAM to visualize the regions contributing
to predictions for visually similar bird classes
introduced in successive tasks.
Activations concentrate on a few local regions,
with different body parts emphasized across tasks.
When ERM favors a small subset of predictive cues,
diversity-driven learning may direct new branches
toward alternative local cues rather than
fine-grained differences in shared attributes. To examine whether these localized responses lead to
cross-task confusion, we analyze classification errors
on CUB-200~\cite{wah2011caltech}.
We group old classes by their maximum semantic similarity
to new classes using attribute prototypes
~\cite{wah2011caltech}.
As shown in Fig.~\ref{motivation_data}(a),
DER exhibits higher Old$\rightarrow$New Error,
the proportion of old-class samples misclassified
as new classes, for classes with greater semantic overlap.
We further examine cross-task feature similarity
using CKA~\cite{kornblith2019similarity}
on semantically similar classes from
ImageNet-1K~\cite{deng2009imagenet}.
As shown in Fig.~\ref{motivation_data}(c),
CPNS increases shallow-feature similarity and decreases
deep-feature similarity relative to DER. This pattern is
consistent with greater reuse of shared visual information
in shallow layers and more distinct deep representations.
Together with the localized Grad-CAM responses,
these findings suggest that feature expansion
favors limited visual cues while failing to preserve
shared semantics, resulting in feature collapse
and cross-task confusion.
Figs.~\ref{motivation_data}(a) and (b) further show that
CPNS reduces old-to-new classification errors across
the semantic-overlap groups.

Motivated by these observations, we argue that
mitigating feature collision requires more than
representation diversity. As shown in
Fig.~\ref{motivation}, this requires intra-task class
discrimination and inter-task feature separability.
To this end, we propose CPNS,
a regularizer based on the Probability of Necessity
and Sufficiency (PNS)~\cite{pearl2009causality,yang2023invariant}.
CPNS evaluates feature necessity and sufficiency
within and across tasks through a dual-scope
counterfactual generator. Minimizing
the corresponding CPNS risks promotes intra-task
class discrimination and reduces inter-task
feature interference.
CPNS can be integrated into existing
expansion-based CIL methods as a plug-and-play
regularizer.

The main contributions of this work are summarized as follows.
\begin{itemize}
\item[$\bullet$] We propose CPNS, a dual-scope PNS-based
regularizer for expansion-based CIL that promotes intra-task
class discrimination and inter-task feature separability.

\item[$\bullet$] We develop a counterfactual generator that uses target-class KNN neighborhoods to guide feature perturbations for CPNS risk estimation. Our theoretical analysis further establishes the identifiability of CPNS, provides guarantees for intra-task class discrimination and inter-task feature separability, and derives a generalization bound for CPNS risk.

\item[$\bullet$] Extensive experiments on three
standard benchmarks (CIFAR-100, ImageNet-100, and
ImageNet-1000) and four fine-grained benchmarks
(CUB200, Birds525, Flower102, and Food101)
demonstrate consistent improvements over
existing expansion-based CIL methods.
\end{itemize}

\section{Related work}

\subsection{Expansion-based CIL}
Model expansion strategies address the stability--plasticity
trade-off by allocating trainable capacity to new tasks while
retaining historical features during expansion
~\cite{yan2021dynamically}.
This separation supports the reuse of prior knowledge and
the acquisition of new class attributes.
Effective expansion therefore requires the added features
to complement historical representations and discriminate
both old and new classes.
In PNN~\cite{rusu2016progressive}, a new backbone is introduced for each task, while frozen features are reused through layer-wise connections. In DER~\cite{yan2021dynamically}, a separate backbone is trained for each task, and the resulting features are aggregated for classification. Similarly, in DyTox~\cite{douillard2022dytox}, a task-specific token is assigned to each task within a Transformer architecture. To balance plasticity and model size, a progression-compression protocol is adopted in PC~\cite{schwarz2018progress}. Following a similar idea, a two-stage strategy is employed in FOSTER~\cite{wang2022foster}, where new modules are first expanded to enhance feature representation, after which redundant parameters are removed through distillation. In addition, an energy-based bi-compatible framework is established in BEEF~\cite{wang2023beef}. In contrast to these methods, which mainly emphasize task-specific features, TagFex~\cite{zheng2025task} is designed to continually capture diverse features through a separate task-agnostic model, and these features are then aggregated with task-specific features via a merge attention mechanism, while its extension BitFea~\cite{zheng2026bi} adds rehearsal feature augmentation to preserve old-class discrimination. Overall, most existing studies adopt diverse feature extraction as the primary strategy for mitigating feature conflicts.

However, this diversity strategy leads to the learning process within each task still being dominated by Empirical Risk Minimization (ERM). ERM tends to prioritize the most accessible discriminative cues, leading to feature suppression, where the model captures only the shortcut features sufficient to minimize training loss. Consequently,  existing expansion-based methods mainly ensure the diversity of the shortcut features for each task of the model. While this reliance on shortcut features temporarily mitigates feature collision, it inherently sacrifices the model's robustness against distribution shifts and constrains the representational depth essential for long-term scalability. Existing expansion-based methods emphasize feature diversity, but diversity alone does not ensure that newly added representations capture the cues needed to distinguish similar classes. CPNS addresses this limitation through counterfactual modeling that promotes intra-task class discrimination and inter-task feature separability.

\subsection{Causal Representation Learning}
Prior work applies PNS to out-of-distribution generalization~\cite{yang2023invariant} and studies the identifiability of causal representations from multimodal or partially observed views~\cite{sun2025causal,yao2024multi}. Interventional and weakly supervised methods use intervention data to identify latent causal factors~\cite{ahuja2023interventional,brehmer2022weakly,JMLR:v26:24-0194,Shirahmad_Gale_Bagi_2026}. CPNS adapts PNS to regularize feature expansion through intra-task and inter-task interventions as new classes arrive. It constructs feature interventions from class-conditional KNN neighborhoods. Our theoretical analysis establishes the identifiability of CPNS, provides guarantees for intra-task class discrimination and inter-task feature separability, and derives a generalization bound for CPNS risk.

\section{Preliminaries}
\subsection{Problem Formulation}
In CIL, a model learns tasks indexed by $t=1,2,\ldots$, with task $1$ as the base task. Let $\mathcal D_t=\{(x_{t,i},y_{t,i})\}_{i=1}^{n}$ be the $n$ training input-label pairs for task $t$, and let $\mathcal C_t$ be its class set. The class sets are disjoint across tasks. Write $\mathcal C_{<t}=\bigcup_{j=1}^{t-1}\mathcal C_j$ and $\mathcal C_{\le t}=\bigcup_{j=1}^{t}\mathcal C_j$ for the historical and cumulative class sets. At task $t$, training uses $\mathcal D_t$ and, in rehearsal-based methods, a fixed-size buffer $\mathcal M$ of past exemplars. Evaluation covers all classes in $\mathcal C_{\le t}$. Write $m=|\mathcal M|$, $\mathcal D=\mathcal D_t\cup\mathcal M$, and $N=n+m$, with $\mathcal M=\varnothing$ when replay is absent.

\subsection{Feature Expansion of Task-Specific Models}
At task $t$, DER~\cite{yan2021dynamically} creates a feature extractor $f_t$ and freezes the previous extractors $f_1,\ldots,f_{t-1}$. Write $f_{\mathrm{old}}(x)=[f_1(x);\ldots;f_{t-1}(x)]$, where $[\,\cdot\,;\,\cdot\,]$ denotes feature concatenation. This historical block is empty at the base task. The main classifier expands to all observed classes and inherits the old classifier weights. For a current-task sample $(x_t,y_t)$, its classification loss is
\begin{equation}
\label{cls}
\mathcal L_{\mathrm{cls}}(x_t,y_t)
=\ell_{\mathrm{CE}}([f_1(x_t);\ldots;f_t(x_t)],y_t;W_{\mathrm{cls}}^{(t)}),
\end{equation}
where $\ell_{\mathrm{CE}}(\mathbf v,y;W)$ is cross-entropy for feature $\mathbf v$, label $y$, and classifier parameters $W$. The main classifier $W_{\mathrm{cls}}^{(t)}$ predicts labels in $\mathcal C_{\le t}$.

The auxiliary classifier uses only $f_t(x)$. Its label map is $\pi_t(y)=y$ for $y\in\mathcal C_t$ and $\pi_t(y)=\mathrm{old}$ for $y\in\mathcal C_{<t}$, where $\mathrm{old}$ is a distinct auxiliary label. At an incremental task, it has $|\mathcal C_t|+1$ outputs, and its loss for a current or replay sample is
\begin{equation}
\mathcal L_{\mathrm{aux}}(x,y)
=\ell_{\mathrm{CE}}(f_t(x),\pi_t(y);W_{\mathrm{aux}}^{(t)}),
\end{equation}
where $W_{\mathrm{aux}}^{(t)}$ denotes the auxiliary classifier parameters. The auxiliary labels retain current-class distinctions and merge all historical classes.

\subsection{Causal Analysis of expansion-based CIL}

\begin{figure}[htbp]
\centering
\includegraphics[width=\columnwidth]{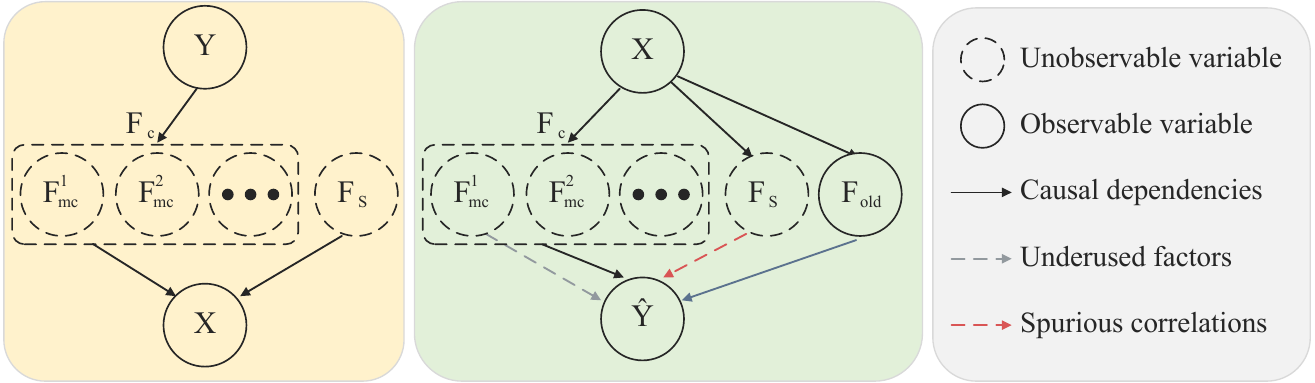}
\caption{Causal data generation and representation
learning in expansion-based CIL.}
\label{SCM}
\end{figure}

As illustrated in Fig.~\ref{SCM} (left),
we consider a class-conditional causal graph
of the data-generation process.
Here, $Y$ and $X$ denote the class label
and the observed sample, respectively.
$F_c$ and $F_s$ represent causal factors
and non-causal background factors.
The label $Y$ determines $F_c$, while
$F_c$ and $F_s$ jointly generate $X$.
We denote by $F_{\mathrm{mc}}^i\subseteq F_c$
a minimal subset of causal factors sufficient
for current-task discrimination. The index $i$ enumerates these subsets.
Multiple such subsets may exist, but each
may omit causal attributes needed to
distinguish classes introduced later.
Although $F_s$ does not cause $Y$,
biased sampling may induce correlations
between them in the observed training
distribution.

Fig.~\ref{SCM} (right) illustrates the
representation learning process in
expansion-based CIL.
Here, $F_{\mathrm{old}}$ denotes the frozen historical
representation, written as $\mathbf z_{\mathrm{old}}=f_{\mathrm{old}}(x)$ below.
In the right panel, $F_c$ and $F_s$
schematically denote the learned representations
of the corresponding generative factors,
rather than the generating factors themselves.
Ideally, the current branch should capture
causal attributes beyond those sufficient
for current-task discrimination.
However, ERM may favor a limited subset of
discriminative causal factors or spurious
background cues, leaving other causal
factors underrepresented
~\cite{pezeshki2021gradient}.
When old and new classes share these cues,
the current and historical representations
may provide ambiguous evidence to the
joint classifier, increasing cross-task
confusion.
Since historical representations are frozen,
previously omitted causal attributes cannot
be recovered within the old feature branch.

Therefore, we aim to improve intra-task class discrimination
by capturing broader causal attributes
beyond those favored by current-task
discrimination, and to improve inter-task
discriminability by reducing ambiguous
cross-task responses while preserving
shared causal semantics.

\subsection{Probability of Necessity and Sufficiency (PNS)}

Probability of Necessity and Sufficiency is used to measure whether a cause variable is both sufficient and necessary for an outcome.

\begin{definition}[\textbf{Probability of Necessity and Sufficiency (PNS)}~\cite{pearl2009causality,yang2023invariant}]
\upshape
\label{def:pns}

Within this definition, \(X\) denotes a binary cause variable with values \(x\) and \(\bar{x}\), and \(Y\) denotes the outcome variable. Then, the probability that \(X\) is a sufficient and necessary cause of \(Y=y\) can be defined as
\begin{equation}
\begin{aligned}
&\mathrm{PNS}(X):= 
\\&\underbrace{P\!\left(Y_{do(X=x)} = y \mid X = \bar{x}, Y \neq y\right)}_{\text{sufficiency}} P(X=\bar{x}, Y\neq y) \\
& + \underbrace{P\!\left(Y_{do(X=\bar{x})} \neq y \mid X = x, Y = y\right)}_{\text{necessity}} P(X=x, Y=y),
\end{aligned}
\label{eq:pns}
\end{equation}
The first term captures sufficiency by asking whether setting $X=x$ produces $Y=y$ for cases observed with $X=\bar{x}$ and $Y\ne y$. The second captures necessity by asking whether setting $X=\bar{x}$ changes the outcome from $y$ for cases observed with $X=x$ and $Y=y$.
\end{definition}

\section{Method}

\subsection{Feature Collapse and Cross-Task Confusion}
\label{sec:expansion_gap}

Expansion-based CIL freezes historical branches to retain learned knowledge. The new branch must supply information needed for subsequent class distinctions. ERM can concentrate learning on a few predictive cues, while diversity-oriented expansion can discourage useful shared attributes. The resulting features may differ from historical features without capturing the details needed to distinguish old and new classes.

\noindent\textbf{Cue omission under ERM.}
Let $\mathbf z_{\mathrm{old}}=f_{\mathrm{old}}(x)$ and $\hat{\mathbf c}=f_t(x)$. Fix an affine classifier on $[\mathbf z_{\mathrm{old}};\hat{\mathbf c}]$, with class-$k$ logit $s_k$ and true-class margin $m(x)=s_y-\max_{k\ne y}s_k$. If its current-feature weights satisfy $\|\mathbf w_k^C\|_2\le W$, then
\begin{equation}
\label{eq:erm_gradient_bound}
\|\nabla_{\hat{\mathbf c}}\ell_{\mathrm{CE}}\|_2
\le2W(|\mathcal C_{\le t}|-1)e^{-m(x)}.
\end{equation}
Once a few cues establish large training margins, the learning signal for additional attributes weakens~\cite{pezeshki2021gradient}. Attributes unnecessary for current-task prediction may still distinguish future fine-grained classes. Freezing the historical branch prevents it from learning these omitted attributes.

\noindent\textbf{Amplification by diversity-oriented expansion.}
Feature diversity does not ensure that the new branch supplies missing distinctions. Compare candidate representations $\mathbf C_s$ retaining fine-grained shared attributes and $\mathbf C_d$ using alternative cues. Let $L_{\mathrm{sup}}(\mathbf C)$ denote the supervised objective, using $\mathcal L_{\mathrm{cls}}$ and $\mathcal L_{\mathrm{aux}}$ for DER, and $D(\mathbf C,\mathbf z_{\mathrm{old}})$ the analytical diversity penalty. Lower $D$ favors greater diversity.

\begin{proposition}[Diversity-induced feature preference]
\label{prop:diversity_preference}
For the analytical objective $L_{\mathrm{exp}}(\mathbf C)=L_{\mathrm{sup}}(\mathbf C)+\beta D(\mathbf C,\mathbf z_{\mathrm{old}})$ with $\beta>0$, if
\begin{equation}
\label{eq:diversity_selection_condition}
\begin{aligned}
&L_{\mathrm{sup}}(\mathbf C_d)-L_{\mathrm{sup}}(\mathbf C_s)\\
&\quad<\beta\bigl[D(\mathbf C_s,\mathbf z_{\mathrm{old}})-D(\mathbf C_d,\mathbf z_{\mathrm{old}})\bigr],
\end{aligned}
\end{equation}
then $L_{\mathrm{exp}}(\mathbf C_d)<L_{\mathrm{exp}}(\mathbf C_s)$.
\end{proposition}
Thus, a small supervised benefit from shared attributes can be outweighed by a preference for alternative cues.

Fine-grained classes differ in shared attributes such as texture, color, and local structure. Expansion that weakens these attributes can leave historical information gaps unfilled and impair old--new discrimination. CPNS addresses this gap through intra-task class discrimination and inter-task feature separability. Derivations are provided in \textbf{Appendix A-I}.

\subsection{Definition of CPNS}

As discussed in Section~\ref{sec:expansion_gap}, a few predictive cues may suffice for current-task classification, while feature diversity alone does not ensure that the new branch captures the fine-grained differences needed to distinguish visually similar old and new classes. We therefore define CPNS to evaluate feature sufficiency and necessity. For a fixed prediction model, CPNS assesses whether factual features support correct predictions and whether replacing them with competing-class counterfactuals invalidates those predictions. The corresponding regularization remains active when factual predictions are correct but counterfactual replacements still predict the true class. Intra-task CPNS evaluates this dependence over current classes. Inter-task CPNS holds historical features fixed during joint classification to assess the new branch's contribution to old--new discrimination.

\begin{definition}[\textbf{Probability of Necessity and Sufficiency in Expansion-based CIL (CPNS)}]
\upshape
\label{Cpns}
Let $(X_t,Y_t)$ denote current-task inputs and true labels, and let $(X,Y)$ denote inputs and true labels from the current-task/replay mixture. Fix the prediction model and feature-generation rules. The feature values depend on the sampled input-label pair. Each intervention replaces the feature while holding this pair fixed.

For a current-task sample, let $\hat{\mathbf c}=f_t(x_t)$ be its factual feature and $\bar{\mathbf c}_{\mathrm{intra}}$ its generated replacement. Let $W_{\mathrm{intra}}$ be the classifier weight matrix mapping current features to logits over $\mathcal C_t$. The feature input satisfies $\mathbf C\in\{\hat{\mathbf c},\bar{\mathbf c}_{\mathrm{intra}}\}$, and the classifier predicts $\widehat Y_t=\sigma(W_{\mathrm{intra}}^\top\mathbf C)$. Here, $\sigma$ returns the largest-logit class with fixed tie breaking. The intra-task probability is
\par
\begingroup
\footnotesize
\begin{equation}
\begin{aligned}
&\mathrm{PNS}_{\mathrm{intra}}(\mathbf C):=\\
&\underbrace{P\!\left((\widehat Y_t)_{do(\mathbf C=\hat{\mathbf c})}=Y_t\mid\mathbf C=\bar{\mathbf c}_{\mathrm{intra}},\widehat Y_t\ne Y_t\right)}_{\text{Intra-task Sufficiency}}P(\mathbf C=\bar{\mathbf c}_{\mathrm{intra}},\widehat Y_t\ne Y_t)\\
&+\underbrace{P\!\left((\widehat Y_t)_{do(\mathbf C=\bar{\mathbf c}_{\mathrm{intra}})}\ne Y_t\mid\mathbf C=\hat{\mathbf c},\widehat Y_t=Y_t\right)}_{\text{Intra-task Necessity}}P(\mathbf C=\hat{\mathbf c},\widehat Y_t=Y_t).
\end{aligned}
\label{eq:cpns_intra}
\end{equation}
\par
\endgroup
\noindent
The first term measures whether restoring $\hat{\mathbf c}$ recovers the true label when the generated feature does not predict it. The second measures whether replacing $\hat{\mathbf c}$ changes a correct prediction.

The inter-task scope requires historical reference samples. For a current-task or replay sample, let $\mathbf z_{\mathrm{old}}=f_{\mathrm{old}}(x)$ and $\hat{\mathbf c}=f_t(x)$. The factual joint feature is $\hat{\mathbf z}=[\mathbf z_{\mathrm{old}};\hat{\mathbf c}]$, and its replacement is $\bar{\mathbf z}=[\mathbf z_{\mathrm{old}};\bar{\mathbf c}_{\mathrm{inter}}]$. Both retain the same historical feature. Let $W_{\mathrm{inter}}$ be the classifier weight matrix mapping joint features to logits over $\mathcal C_{\le t}$. The joint feature input satisfies $\mathbf Z\in\{\hat{\mathbf z},\bar{\mathbf z}\}$, and the classifier predicts $\widehat Y=\sigma(W_{\mathrm{inter}}^\top\mathbf Z)$. The inter-task probability is
\par
\begingroup
\small
\begin{equation}
\begin{aligned}
&\mathrm{PNS}_{\mathrm{inter}}(\mathbf Z):=\\
&\underbrace{P\!\left(\widehat Y_{do(\mathbf Z=\hat{\mathbf z})}=Y\mid\mathbf Z=\bar{\mathbf z},\widehat Y\ne Y\right)}_{\text{Inter-task Sufficiency}}P(\mathbf Z=\bar{\mathbf z},\widehat Y\ne Y)\\
&+\underbrace{P\!\left(\widehat Y_{do(\mathbf Z=\bar{\mathbf z})}\ne Y\mid\mathbf Z=\hat{\mathbf z},\widehat Y=Y\right)}_{\text{Inter-task Necessity}}P(\mathbf Z=\hat{\mathbf z},\widehat Y=Y).
\end{aligned}
\label{eq:cpns_inter}
\end{equation}
\par
\endgroup
\noindent
The first term measures whether restoring the factual current feature recovers the true label. The second measures whether replacing it changes a correct prediction while retaining the historical feature.
For DER, $W_{\mathrm{inter}}=W_{\mathrm{cls}}^{(t)}$,
 and $W_{\mathrm{intra}}$ consists of the current-class columns of $W_{\mathrm{aux}}^{(t)}$.
  For scope $s\in\{\mathrm{intra},\mathrm{inter}\}$, $P$ denotes the joint law $P_s$ of the 
  input-label pair and its selected feature. 
\end{definition}

\subsection{Interventional Characterization of CPNS}

CPNS measures the probability that a factual feature yields a correct prediction while its generated replacement does not. We relate this joint event to prediction probabilities under feature interventions.

Exogeneity allows interventional outcome probabilities to be expressed as observational conditional probabilities. Under individual-level monotonicity, PNS equals the difference between the corresponding interventional outcome probabilities, without requiring exogeneity~\cite{pearl2009causality}. We apply this result to the intra-task and inter-task interventions in Definition~\ref{Cpns}. Each intervention replaces the feature input of a fixed prediction model while holding the sample and its true label fixed.

\begin{samepage}
\begin{theorem}[\textbf{CPNS under Monotonicity}]
\label{theorem:indpns}
Fix the prediction model and feature-generation rules. Suppose that, for almost every input-label pair in each scope, a correct prediction with the generated feature implies a correct prediction with the factual feature. Then the intra-task and inter-task PNS are given by
\begin{equation}
\label{eq:indintra}
\begin{aligned}
\mathrm{PNS}_{\mathrm{intra}}(\mathbf C)={}&
P\!\left((\widehat Y_t)_{do(\mathbf C=\hat{\mathbf c})}=Y_t\right)\\
&{}-P\!\left((\widehat Y_t)_{do(\mathbf C=\bar{\mathbf c}_{\mathrm{intra}})}=Y_t\right),
\end{aligned}
\end{equation}
and
\begin{equation}
\label{eq:indinter}
\begin{aligned}
\mathrm{PNS}_{\mathrm{inter}}(\mathbf Z)={}&
P\!\left(\widehat Y_{do(\mathbf Z=\hat{\mathbf z})}=Y\right)\\
&{}-P\!\left(\widehat Y_{do(\mathbf Z=\bar{\mathbf z})}=Y\right).
\end{aligned}
\end{equation}
\end{theorem}
\end{samepage}

Each probability is taken over the input-label distribution of its scope. For inter-task interventions, $\hat{\mathbf z}$ and $\bar{\mathbf z}$ retain the same historical feature block.

Under the stated monotonicity condition, CPNS can be estimated as the difference between factual and generated prediction accuracies, evaluated on paired features of the same samples with the same classifier. The proof is provided in \textbf{Appendix A-II}.

\subsection{Measurement of CPNS Risk}

We estimate CPNS risk by constructing counterfactual features within and across tasks, then evaluating the classifier predictions on paired factual and generated representations.

\subsubsection{Dual-Scope Counterfactual Feature Generation}
\label{sec:cpns_generation}

Counterfactual generation is guided by validity, minimality, and plausibility~\cite{wachter2017counterfactual,ChenCounterfactualSamples,pmlr-v97-goyal19a,Luu_2025_WACV,ijcai2026p65}. Validity requires classification as a non-factual target class. Minimality favors the smallest feature change that achieves this prediction change. Plausibility requires consistency with the target-class semantics. Since continual learning repeatedly constructs counterfactuals across tasks and classes, we generate them directly in feature space to avoid the cost of diffusion-based iterative sampling~\cite{augustin2022diffusion,pmlr-v267-rasal25a}.

\noindent\textbf{Reference sets and support distance.}
We construct class-conditional reference sets to characterize the empirical support of each target class. For intra-task generation, $\mathcal D^t_{\mathrm{intra},k}$ contains the factual features $f_t(x_i)$ of current-task training samples with label $k\in\mathcal C_t$. For inter-task generation, $\mathcal D^t_{\mathrm{inter},k}$ contains the joint representations $[f_{\mathrm{old}}(x_i);f_t(x_i)]$, where $k\in\mathcal C_{\le t}$. References for current classes come from current-task training samples, while references for historical classes come from replay exemplars. All references are encoded using the same frozen historical extractor and current feature extractor.

For neighborhood computation, we normalize each feature as
\[
\nu(\mathbf v)=\frac{\mathbf v}{\max(\|\mathbf v\|_2,\epsilon)},
\qquad \epsilon>0.
\]
Let $\mathcal N_K(\mathbf v;\mathcal Q)$ contain the $K$ nearest references in a class-conditional set $\mathcal Q$, measured by normalized Euclidean distance. We define the local support radius as
\begin{equation}
\label{eq:knn_support_radius}
r_{K,\mathcal{Q}}(\mathbf{v})
=
\max_{\mathbf{q}\in\mathcal{N}_{K}(\mathbf{v};\mathcal{Q})}
\left\|
\nu(\mathbf v)-\nu(\mathbf q)
\right\|_2,
\end{equation}
This radius is the distance to the $K$-th nearest reference feature. We use this radius both to select a nearby competing class and to guide perturbations toward its reference neighborhood.

\noindent\textbf{Target-class selection.}
For intra-task generation, we select the competing current class whose reference neighborhood is closest to the factual feature, giving
\begin{equation}
\label{eq:intra_knn_target}
a_{\text{intra}}^\star
=
\arg\min_{a\in\mathcal{C}_t\setminus\{y_t\}}
r_{K,\mathcal{D}^{t}_{\text{intra},a}}(\hat{\mathbf{c}}).
\end{equation}

For inter-task generation, we define the admissible target set from the opposite historical/current class partition as
\begin{equation}
\label{eq:cross_task_candidates}
\mathcal{A}_{\mathrm{inter}}(y)
=
\begin{cases}
\mathcal{C}_{<t}, & y\in\mathcal{C}_t,\\
\mathcal{C}_t, & y\in\mathcal{C}_{<t}.
\end{cases}
\end{equation}
The target with the smallest support radius at the factual joint representation is
\begin{equation}
\label{eq:inter_knn_target}
a_{\mathrm{inter}}^\star
=
\arg\min_{a\in\mathcal{A}_{\mathrm{inter}}(y)}
r_{K,\mathcal{D}^{t}_{\mathrm{inter},a}}(\hat{\mathbf{z}}).
\end{equation}
The selected target remains fixed during perturbation search.

\begin{definition}[\textbf{Dual-scope counterfactual perturbation}]
\upshape
\label{Definition:Double-counterfactual}
Let $\tau_{\mathrm{intra}},\tau_{\mathrm{inter}}\in(0,2]$ denote the support-radius thresholds. For each scope $s\in\{\mathrm{intra},\mathrm{inter}\}$, let $\mathcal S_s$ contain the searched perturbations with finite objective values, including $\mathbf0$. The intra-task perturbation is
\begin{equation}
\label{Double-counterfactual1}
\begin{aligned}
\boldsymbol{\delta}^\star_{\text{intra}}
&=\arg\min_{\boldsymbol{\delta}\in\mathcal S_{\mathrm{intra}}}\Big\{\|\boldsymbol{\delta}\|_2^2\\
&\quad+\operatorname{softplus}^2\!\Big[
r_{K,\mathcal{D}^{t}_{\text{intra},a_{\text{intra}}^\star}}
(\hat{\mathbf{c}}+\boldsymbol{\delta})
-\tau_{\text{intra}}\Big]\Big\}.
\end{aligned}
\end{equation}
The inter-task perturbation is
\begin{equation}
\label{Double-counterfactual2}
\begin{aligned}
\boldsymbol{\delta}^{\star}_{\mathrm{inter}}
&=\arg\min_{\boldsymbol{\delta}\in\mathcal S_{\mathrm{inter}}}\Big\{\|\boldsymbol{\delta}\|_2^2\\
&\quad+\operatorname{softplus}^2\!\Big[
r_{K,\mathcal{D}^{t}_{\mathrm{inter},a_{\mathrm{inter}}^\star}}
(\hat{\mathbf z}+[\mathbf0;\boldsymbol\delta])
-\tau_{\mathrm{inter}}\Big]\Big\}.
\end{aligned}
\end{equation}
Here, $\operatorname{softplus}(u)=\log(1+e^u)$. The generator returns
\[
\bar{\mathbf c}_s=\hat{\mathbf c}+\boldsymbol\delta_s^\star.
\]
For inter-task generation, the joint representation retains the historical feature block and takes the form
\[
\bar{\mathbf z}=[\mathbf z_{\mathrm{old}};\bar{\mathbf c}_{\mathrm{inter}}].
\]
\end{definition}

The squared perturbation norm promotes minimality by penalizing large feature changes. The support penalty encourages proximity to the target-class reference neighborhood, providing an empirical criterion for plausibility in feature space. The competing target specifies the desired prediction change. Validity is assessed by whether the generated representation is classified as that target.

\noindent\textbf{Perturbation search.}
Let $\mathcal J_s(\boldsymbol\delta)$ denote the objective for scope $s$ in Eq.~\eqref{Double-counterfactual1} or~\eqref{Double-counterfactual2}. Starting from $\boldsymbol\delta=\mathbf0$, we perform $T$ gradient steps with step size $\eta_{\mathrm{cf}}>0$, keeping the model, reference features, and target class fixed. At each step, we recompute the $K$ nearest references and hold their indices fixed during differentiation. Gradients propagate through the selected distances.

We retain the searched perturbations with finite objective values and return the one with the lowest $\mathcal J_s$, including the zero initialization among the candidates. Ties are resolved by selecting the earliest iterate.

\subsubsection{CPNS Risk Estimation}

\begin{definition}[\textbf{CPNS risk}]
\label{CPNS-risk}
Index the $n$ current-task samples by $i$ and the $N$ samples in
$\mathcal D$ by $k$. In the intra-task average, write
$\bar{\mathbf c}_i=\bar{\mathbf c}_{\mathrm{intra},i}$.
The empirical CPNS risk is
\par
\begingroup
\small
\begin{equation}
\label{eq:cpns_risk}
\begin{aligned}
\hat{R}_{\text{CPNS}} &= \frac{1}{n} \sum_{i=1}^n \Big[ \rho\big(\sigma(W_{\text{intra}}^\top \hat{\mathbf{c}}_i) \neq y_{t,i}\big) + \rho\big(\sigma(W_{\text{intra}}^\top \bar{\mathbf{c}}_i) = y_{t,i}\big) \Big] \\
&+ \frac{1}{N} \sum_{k=1}^N \Big[ \rho\big(\sigma(W_{\text{inter}}^\top \hat{\mathbf{z}}_k) \neq y_k\big) + \rho\big(\sigma(W_{\text{inter}}^\top \bar{\mathbf{z}}_k) = y_k\big) \Big].
\end{aligned}
\end{equation}
\endgroup
Here $\rho$ is the indicator function. The two averages are
$\hat R_{\mathrm{intra}}$ and $\hat R_{\mathrm{inter}}$.
The first term in each bracket penalizes factual misclassification.
The second penalizes generated features that retain the true-label prediction.
\end{definition}

\subsection{Training Objective and Optimization}
\label{sec:optimization}

\noindent\textbf{Margin-based CPNS loss.}
The indicator risk in Eq.~\eqref{eq:cpns_risk} cannot be optimized directly by gradient descent.
We use a margin surrogate to encourage correct factual predictions and a change from the true label after feature replacement.
Both scopes share the same loss form.
For $s\in\{\mathrm{intra},\mathrm{inter}\}$, write the factual and generated pair
$(\mathbf v_s,\bar{\mathbf v}_s)$ as $(\hat{\mathbf c},\bar{\mathbf c}_{\mathrm{intra}})$
or $(\hat{\mathbf z},\bar{\mathbf z})$, respectively.
Let $\mathcal C_{\mathrm{intra}}=\mathcal C_t$, $\mathcal C_{\mathrm{inter}}=\mathcal C_{\le t}$,
and $\mathbf w_{s,j}$ be column $j$ of $W_s$.
The true-class margin over the strongest competitor is
\begin{equation}
\label{eq:cpns_margin}
m_s(\mathbf v,y)
=\mathbf w_{s,y}^\top\mathbf v
-\max_{j\in\mathcal C_s\setminus\{y\}}\mathbf w_{s,j}^\top\mathbf v.
\end{equation}
For $u=(x,y)$ and $\phi(v)=\operatorname{softplus}(v)/\log2$, define
\begin{equation}
\label{eq:cpns_training_surrogate}
\widetilde\ell_s(u)
=\phi\bigl(-m_s(\mathbf v_s,y)\bigr)
+\phi\bigl(m_s(\bar{\mathbf v}_s,y)\bigr).
\end{equation}
The first term promotes a positive factual margin, while the second promotes a negative margin after replacement.
Target classes guide counterfactual generation in Section~\ref{sec:cpns_generation}.
The CPNS loss tests whether the replacement changes the true-label prediction.
Since $\phi$ is nonnegative and increasing with $\phi(0)=1$, the surrogate upper-bounds the two-term indicator loss $\ell_s(u)$
\begin{equation}
\label{eq:cpns_surrogate_upper_bound}
\ell_s(u)\le\widetilde\ell_s(u).
\end{equation}
We average $\widetilde\ell_s$ over current-task samples for $\widehat{\widetilde R}_{\mathrm{intra}}$
and over current and replay samples for $\widehat{\widetilde R}_{\mathrm{inter}}$,
using the weights in Eq.~\eqref{eq:cpns_risk}.

\noindent\textbf{Training objective and schedule.}
Let $\mathcal L_{\mathrm{base}}$ denote the host training loss, including the main and auxiliary classification losses for DER.
At epoch $e$, we optimize
\begin{equation}
\label{eq:total_objective}
\begin{aligned}
\mathcal L_e
&=\mathcal L_{\mathrm{base}}
+\widehat{\widetilde R}_{\mathrm{intra}}
+\lambda_e\widehat{\widetilde R}_{\mathrm{inter}},\\
\lambda_e
&=\begin{cases}
0,&e\le E_1,\\
\lambda,&e>E_1,
\end{cases}
\qquad \lambda>0.
\end{aligned}
\end{equation}
The two CPNS objectives can produce opposing gradients in $f_t$ before current-class boundaries stabilize.
The first $E_1$ epochs establish current-class distinctions, after which inter-task regularization is activated with weight $\lambda$.
Both stages update the trainable host parameters, $f_t$, and $W_{\mathrm{intra}}$.
The second stage adds inter-task CPNS gradients to $f_t$ and $W_{\mathrm{inter}}$.
For DER, $W_{\mathrm{inter}}$ is the main classifier and receives gradients from $\mathcal L_{\mathrm{base}}$ in both stages.
The historical extractor $f_{\mathrm{old}}$ remains frozen.
We omit the inter-task term when historical references are unavailable.

\noindent\textbf{Counterfactual search and model update.}
We refresh class-conditional reference features at each epoch.
Each mini-batch has two steps.
First, we fix the model and reference features and search for perturbations in the active scopes as in Section~\ref{sec:cpns_generation}.
Second, we recompute current features and update the model under $\mathcal L_e$, treating the selected perturbations as constants.
For the parameters $\theta_t$ of $f_t$, the generated current feature is
\begin{equation}
\label{eq:detached_perturbation}
\bar{\mathbf c}_s(\theta_t)
=f_t(x;\theta_t)+\operatorname{sg}(\boldsymbol\delta_s^\star),
\end{equation}
Here $\operatorname{sg}$ stops gradients through the perturbation.
The generated feature still propagates gradients through $f_t$ and the classification heads.
During the model update, gradients do not pass through historical features, reference features, target selection, or the perturbation search.

\subsection{Theoretical Analysis}
\label{sec:theory}

We establish CPNS identifiability, its constraints on intra-task separability and inter-task discriminability, and its generalization guarantee.

Let $\mu_t$ denote the current-task input--label distribution and $\alpha=n/N$. The intra-task distribution is $\mu_{\mathrm{intra}}=\mu_t$. With a nonempty replay buffer, the inter-task distribution is $\mu_{\mathrm{inter}}=\alpha\mu_t+(1-\alpha)\widehat\mu_{\mathcal M}$, where $\widehat\mu_{\mathcal M}$ is the uniform replay-buffer distribution.

\begingroup
\makeatletter
\def\@IEEEsectpunct{.\ }
\makeatother
\subsubsection{Identifiability of CPNS}
We establish identifiability through the monotonicity condition in Theorem~\ref{theorem:indpns}.
This condition requires that a correct generated prediction implies a correct factual prediction.
For scope $s\in\{\mathrm{intra},\mathrm{inter}\}$ and $u=(x,y)$, define the paired correctness indicators
\begin{equation}
\label{eq:paired_correctness}
\begin{aligned}
V_s^1(u)&=\rho\!\left(\sigma(W_s^\top\mathbf v_s)=y\right),\\
V_s^0(u)&=\rho\!\left(\sigma(W_s^\top\bar{\mathbf v}_s)=y\right),
\end{aligned}
\end{equation}
where $(\mathbf v_s,\bar{\mathbf v}_s)$ follows Section~\ref{sec:optimization}.
Under $\mu_s$, the CPNS in Definition~\ref{Cpns} is $\mathrm{PNS}_s=\mathbb P_{\mu_s}(V_s^1=1,V_s^0=0)$, and the accuracy difference is $\Delta_s=\mathbb E_{\mu_s}[V_s^1-V_s^0]$.
Monotonicity requires $V_s^0\le V_s^1$ almost surely.
Its violation probability is $M_s=\mathbb P_{\mu_s}(V_s^1=0,V_s^0=1)$.
The associated loss and risk are $\ell_s=1-V_s^1+V_s^0$ and $R_s=\mathbb E_{\mu_s}[\ell_s]$.

\begin{proposition}[Monotonicity and CPNS identifiability]
\label{prop:risk_relation}
For either scope,
\begin{equation}
\label{eq:cpns_risk_decomposition}
\begin{gathered}
0\le\mathrm{PNS}_s-\Delta_s=M_s\le R_s/2,\\
R_s=1-\mathrm{PNS}_s+M_s.
\end{gathered}
\end{equation}
Let $\hat M_s$ denote the empirical violation frequency. The training surrogate satisfies
\begin{equation}
\label{eq:surrogate_monotonicity_bound}
2\hat M_s\le\hat R_s\le\widehat{\widetilde R}_s.
\end{equation}
Zero risk $R_s=0$ ensures monotonicity.
Under monotonicity, $M_s=0$ and $\mathrm{PNS}_s=\Delta_s=1-R_s$.
\end{proposition}

The risk controls monotonicity violations, and the training surrogate provides an empirical upper bound.
Under monotonicity, the factual--generated accuracy difference identifies CPNS.
The proof is provided in \textbf{Appendix A-II}.

\subsubsection{Intra-task separability and inter-task discriminability}
Section~\ref{sec:expansion_gap} bounds ERM gradients by a term that decays with the factual margin and shows that diversity can favor features with insufficient old--new discrimination.
CPNS addresses these gaps through paired factual and counterfactual margins.
When a replacement still predicts the true class, its term in Eq.~\eqref{eq:cpns_training_surrogate} is at least one, keeping the constraint active at large factual margins.
The following theorem translates this requirement into intra-task separability and inter-task discriminability.

Fix a loss threshold $\eta\in(0,1)$ and let $\gamma_\eta=-\log(2^\eta-1)>0$.

\begin{theorem}[Separability and discriminability under CPNS]
\label{thm:cpns_discrimination_separability}
For either scope, $\widetilde\ell_s(u)\le\eta$ implies
\begin{equation}
\label{eq:cpns_signed_margin_certificate}
m_s(\mathbf v_s,y)\ge\gamma_\eta,
\qquad
m_s(\bar{\mathbf v}_s,y)\le-\gamma_\eta.
\end{equation}
Consider two samples $i,j$ with loss at most $\eta$ in the same scope.

\noindent\textnormal{(i) Intra-task separability.}
If they belong to different current classes and the classifier weights satisfy $\|\mathbf w_{\mathrm{intra},k}\|_2\le B_{\mathrm{intra}}$ with $B_{\mathrm{intra}}>0$, then
\begin{equation}
\label{eq:cpns_intra_feature_separation}
\|\hat{\mathbf c}_i-\hat{\mathbf c}_j\|_2
\ge\frac{\gamma_\eta}{B_{\mathrm{intra}}}.
\end{equation}

\noindent\textnormal{(ii) Inter-task discriminability.}
For one historical and one current sample, if the joint weights $\mathbf w_{\mathrm{inter},k}=[\mathbf w_k^F;\mathbf w_k^C]$ satisfy $\|\mathbf w_k^F\|_2\le B_F$ and $\|\mathbf w_k^C\|_2\le B_C$ with $B_F\ge0$ and $B_C>0$, then
\begin{equation}
\label{eq:cpns_inter_feature_separation}
\|\hat{\mathbf c}_i-\hat{\mathbf c}_j\|_2
\ge\frac{\bigl[\gamma_\eta-B_F\|\mathbf z_{\mathrm{old},i}-\mathbf z_{\mathrm{old},j}\|_2\bigr]_+}{B_C},
\end{equation}
where $[a]_+=\max\{a,0\}$.
\end{theorem}

Equation~\eqref{eq:cpns_inter_feature_separation} addresses the historical information gap in Section~\ref{sec:expansion_gap}.
When old and new samples have similar frozen features, the current branch must supply the remaining separation required by the classification margin.

For the empirical distribution $\widehat\mu_s$ used in $\widehat{\widetilde R}_s$, the fraction of samples outside the low-loss set satisfies
\begin{equation}
\label{eq:cpns_low_margin_fraction}
\widehat\mu_s\bigl(\widetilde\ell_s>\eta\bigr)
\le\min\left\{1,\frac{\widehat{\widetilde R}_s}{\eta}\right\}.
\end{equation}
A smaller empirical surrogate risk strengthens the coverage guarantee for both constraints.
Proofs are provided in \textbf{Appendix A-III}.

\subsubsection{Generalization of CPNS}
We bound the expected CPNS risk using the empirical training surrogate.
Condition on the historical information, including $f_{\mathrm{old}}$ and $\mathcal M$, and let $\mathcal D_t$ contain $n$ independent samples from $\mu_t$.
For each scope, let $\mathcal H_s$ and $\mathfrak Q_s$ be the model and class-indexed reference families with prescribed dimensions and capacity bounds.
These families and the search settings are fixed before drawing $\mathcal D_t$.

Writing $\ell_s(h,u;q)$ for the indicator loss under model $h$ and reference collection $q$, define
\begin{equation}
\label{eq:loss_class}
\mathcal F_s=\{u\mapsto\ell_s(h,u;q)\mid
h\in\mathcal H_s,\ q\in\mathfrak Q_s\}.
\end{equation}
Each loss lies in $[0,2]$. Its expected Rademacher complexity~\cite{yin2019rademacher} is
\begin{equation}
\label{eq:rademacher}
\mathfrak R_n(\mathcal F_s)
=\mathbb E_{\mathcal D_t,\boldsymbol\sigma}
\left[\sup_{h,q}\frac1n\sum_{i=1}^n
\sigma_i\ell_s(h,u_i;q)\right].
\end{equation}
Here $u_i=(x_{t,i},y_{t,i})$, and the $\sigma_i$ are independent uniform signs in $\{-1,1\}$.
Expectations are conditional on the historical information.

For fixed $\lambda\ge0$, define $R_\lambda=R_{\mathrm{intra}}+\lambda R_{\mathrm{inter}}$ and $\hat R_\lambda=\hat R_{\mathrm{intra}}+\lambda\hat R_{\mathrm{inter}}$, so $\hat R_1=\hat R_{\mathrm{CPNS}}$. The weighted surrogate is $\widehat{\widetilde R}_\lambda=\widehat{\widetilde R}_{\mathrm{intra}}+\lambda\widehat{\widetilde R}_{\mathrm{inter}}$.

\begin{theorem}[Generalization of CPNS risk]
\label{thm:generalization}
For any $\delta\in(0,1)$, with probability at least $1-\delta$ over $\mathcal D_t$, uniformly over admissible models and reference collections,
\begin{equation}
\label{eq:generalization}
\begin{aligned}
R_\lambda&\le\hat R_\lambda+\mathcal B_\delta\le\widehat{\widetilde R}_\lambda+\mathcal B_\delta,
\end{aligned}
\end{equation}
where
\begin{equation}
\label{eq:complexity_term}
\begin{aligned}
\mathcal B_\delta
={}&2\mathfrak R_n(\mathcal F_{\mathrm{intra}})
+2\lambda\alpha\mathfrak R_n(\mathcal F_{\mathrm{inter}})\\
&+2(1+\lambda\alpha)
\sqrt{\frac{\log(2/\delta)}{2n}}.
\end{aligned}
\end{equation}
Consequently,
\begin{equation}
\label{eq:pns_deficit_bound}
\begin{aligned}
&(1-\mathrm{PNS}_{\mathrm{intra}})
+\lambda(1-\mathrm{PNS}_{\mathrm{inter}})\le\widehat{\widetilde R}_\lambda+\mathcal B_\delta.
\end{aligned}
\end{equation}
\end{theorem}
The theorem provides a generalization guarantee by bounding both expected CPNS risk and weighted PNS deficit.
The term $\mathcal B_\delta$ accounts for model and reference complexity and finite-sample uncertainty.
The proof is provided in \textbf{Appendix A-IV}.
\endgroup

\begingroup
\flushbottom
\setlength{\textfloatsep}{10pt plus 2pt minus 2pt}
\setlength{\floatsep}{8pt plus 2pt minus 2pt}
\setlength{\intextsep}{8pt plus 2pt minus 2pt}
\section{Experiments}

\begin{table*}[!b]
\begingroup
\centering
\caption{Performance results. Parentheses show CPNS gains over the corresponding baseline in percentage points.}
\label{tab:results}
\fontsize{8}{9.5}\selectfont
\setlength{\tabcolsep}{1.5pt}
\renewcommand{\arraystretch}{1.12}
\setlength{\aboverulesep}{0.4ex}
\setlength{\belowrulesep}{0.55ex}
\begin{tabular}{
>{\raggedright\arraybackslash}p{\dimexpr0.10\textwidth-2\tabcolsep\relax}
*{10}{>{\centering\arraybackslash}p{\dimexpr0.09\textwidth-2\tabcolsep\relax}}}
\toprule
\multirow{3}{*}{Method} & \multicolumn{4}{c}{CIFAR-100} & \multicolumn{4}{c}{ImageNet-100} & \multicolumn{2}{c}{ImageNet-1000} \\ 
\cmidrule(lr){2-5} \cmidrule(lr){6-9} \cmidrule(lr){10-11}
& \multicolumn{2}{c}{10-10} & \multicolumn{2}{c}{50-10} & \multicolumn{2}{c}{10-10} & \multicolumn{2}{c}{50-10} & \multicolumn{2}{c}{100-100} \\ 
\cmidrule(lr){2-3} \cmidrule(lr){4-5} \cmidrule(lr){6-7} \cmidrule(lr){8-9} \cmidrule(lr){10-11}
& Last & Avg & Last & Avg & Last & Avg & Last & Avg & Last & Avg \\ 
\midrule
DER & 64.35 & 75.36 & 65.27 & 72.60 & 66.71 & 77.18 & 71.08 & 77.71 & 58.83 & 66.87 \\
\rowcolor{lightgray}
w/ CPNS & \textbf{66.53}\,{\fontsize{6.5}{7.5}\selectfont(\textcolor{red}{+2.18})} & \textbf{77.08}\,{\fontsize{6.5}{7.5}\selectfont(\textcolor{red}{+1.72})} & \textbf{67.69}\,{\fontsize{6.5}{7.5}\selectfont(\textcolor{red}{+2.42})} & \textbf{74.40}\,{\fontsize{6.5}{7.5}\selectfont(\textcolor{red}{+1.80})} & \textbf{68.51}\,{\fontsize{6.5}{7.5}\selectfont(\textcolor{red}{+1.80})} & \textbf{78.46}\,{\fontsize{6.5}{7.5}\selectfont(\textcolor{red}{+1.28})} & \textbf{72.88}\,{\fontsize{6.5}{7.5}\selectfont(\textcolor{red}{+1.80})} & \textbf{79.49}\,{\fontsize{6.5}{7.5}\selectfont(\textcolor{red}{+1.78})} & \textbf{60.45}\,{\fontsize{6.5}{7.5}\selectfont(\textcolor{red}{+1.62})} & \textbf{67.86}\,{\fontsize{6.5}{7.5}\selectfont(\textcolor{red}{+0.99})} \\ 
\midrule

FOSTER & 62.20 & 74.49 & 59.80 & 67.54 & 65.68 & 76.74 & 71.60 & 77.37 & 59.03 & 67.18 \\
\rowcolor{lightgray}
w/ CPNS & \textbf{64.32}\,{\fontsize{6.5}{7.5}\selectfont(\textcolor{red}{+2.12})} & \textbf{75.69}\,{\fontsize{6.5}{7.5}\selectfont(\textcolor{red}{+1.20})} & \textbf{61.42}\,{\fontsize{6.5}{7.5}\selectfont(\textcolor{red}{+1.62})} & \textbf{69.06}\,{\fontsize{6.5}{7.5}\selectfont(\textcolor{red}{+1.52})} & \textbf{67.16}\,{\fontsize{6.5}{7.5}\selectfont(\textcolor{red}{+1.48})} & \textbf{77.97}\,{\fontsize{6.5}{7.5}\selectfont(\textcolor{red}{+1.23})} & \textbf{72.95}\,{\fontsize{6.5}{7.5}\selectfont(\textcolor{red}{+1.35})} & \textbf{78.58}\,{\fontsize{6.5}{7.5}\selectfont(\textcolor{red}{+1.21})} & \textbf{60.37}\,{\fontsize{6.5}{7.5}\selectfont(\textcolor{red}{+1.34})} & \textbf{67.83}\,{\fontsize{6.5}{7.5}\selectfont(\textcolor{red}{+0.65})} \\ 
\midrule

BEEF & 60.98 & 71.94 & 63.51 & 70.71 & 68.78 & 77.62 & 70.98 & 77.27 & 58.67 & 67.09 \\
\rowcolor{lightgray}
w/ CPNS & \textbf{62.72}\,{\fontsize{6.5}{7.5}\selectfont(\textcolor{red}{+1.74})} & \textbf{73.52}\,{\fontsize{6.5}{7.5}\selectfont(\textcolor{red}{+1.58})} & \textbf{64.93}\,{\fontsize{6.5}{7.5}\selectfont(\textcolor{red}{+1.42})} & \textbf{72.21}\,{\fontsize{6.5}{7.5}\selectfont(\textcolor{red}{+1.50})} & \textbf{70.33}\,{\fontsize{6.5}{7.5}\selectfont(\textcolor{red}{+1.55})} & \textbf{78.98}\,{\fontsize{6.5}{7.5}\selectfont(\textcolor{red}{+1.36})} & \textbf{72.42}\,{\fontsize{6.5}{7.5}\selectfont(\textcolor{red}{+1.44})} & \textbf{78.60}\,{\fontsize{6.5}{7.5}\selectfont(\textcolor{red}{+1.33})} & \textbf{60.04}\,{\fontsize{6.5}{7.5}\selectfont(\textcolor{red}{+1.37})} & \textbf{68.21}\,{\fontsize{6.5}{7.5}\selectfont(\textcolor{red}{+1.12})} \\ 
\midrule


BitFea & 69.53 & 79.65 & 71.83 & 76.62 & 71.90 & 79.16 & 75.13 & 81.41 & 62.31 & 68.98 \\
\rowcolor{lightgray}
w/ CPNS & \textbf{70.78}\,{\fontsize{6.5}{7.5}\selectfont(\textcolor{red}{+1.25})} & \textbf{79.89}\,{\fontsize{6.5}{7.5}\selectfont(\textcolor{red}{+0.24})} & \textbf{72.48}\,{\fontsize{6.5}{7.5}\selectfont(\textcolor{red}{+0.65})} & \textbf{77.14}\,{\fontsize{6.5}{7.5}\selectfont(\textcolor{red}{+0.52})} & \textbf{72.62}\,{\fontsize{6.5}{7.5}\selectfont(\textcolor{red}{+0.72})} & \textbf{80.10}\,{\fontsize{6.5}{7.5}\selectfont(\textcolor{red}{+0.94})} & \textbf{76.45}\,{\fontsize{6.5}{7.5}\selectfont(\textcolor{red}{+1.32})} & \textbf{81.98}\,{\fontsize{6.5}{7.5}\selectfont(\textcolor{red}{+0.57})} & \textbf{63.21}\,{\fontsize{6.5}{7.5}\selectfont(\textcolor{red}{+0.90})} & \textbf{70.42}\,{\fontsize{6.5}{7.5}\selectfont(\textcolor{red}{+1.44})} \\ 
\midrule

CPrompt & 87.50 & 91.96 & 88.23 & 90.95 & 93.62 & 95.58 & 93.92 & 95.25 & 79.53 & 88.33 \\
\rowcolor{lightgray}
w/ CPNS-intra & \textbf{87.99}\,{\fontsize{6.5}{7.5}\selectfont(\textcolor{red}{+0.49})} & \textbf{92.44}\,{\fontsize{6.5}{7.5}\selectfont(\textcolor{red}{+0.48})} & \textbf{88.61}\,{\fontsize{6.5}{7.5}\selectfont(\textcolor{red}{+0.38})} & \textbf{91.62}\,{\fontsize{6.5}{7.5}\selectfont(\textcolor{red}{+0.67})} & \textbf{94.03}\,{\fontsize{6.5}{7.5}\selectfont(\textcolor{red}{+0.41})} & \textbf{96.24}\,{\fontsize{6.5}{7.5}\selectfont(\textcolor{red}{+0.66})} & \textbf{94.38}\,{\fontsize{6.5}{7.5}\selectfont(\textcolor{red}{+0.46})} & \textbf{95.86}\,{\fontsize{6.5}{7.5}\selectfont(\textcolor{red}{+0.61})} & \textbf{80.61}\,{\fontsize{6.5}{7.5}\selectfont(\textcolor{red}{+1.08})} & \textbf{89.49}\,{\fontsize{6.5}{7.5}\selectfont(\textcolor{red}{+1.16})} \\
\midrule

CL-LoRA & 88.32 & 92.39 & 87.43 & 90.11 & 92.12 & 95.06 & 87.94 & 91.18 & 76.14 & 86.10 \\
\rowcolor{lightgray}
w/ CPNS-intra & \textbf{88.95}\,{\fontsize{6.5}{7.5}\selectfont(\textcolor{red}{+0.63})} & \textbf{92.76}\,{\fontsize{6.5}{7.5}\selectfont(\textcolor{red}{+0.37})} & \textbf{88.19}\,{\fontsize{6.5}{7.5}\selectfont(\textcolor{red}{+0.76})} & \textbf{90.64}\,{\fontsize{6.5}{7.5}\selectfont(\textcolor{red}{+0.53})} & \textbf{92.81}\,{\fontsize{6.5}{7.5}\selectfont(\textcolor{red}{+0.69})} & \textbf{95.99}\,{\fontsize{6.5}{7.5}\selectfont(\textcolor{red}{+0.93})} & \textbf{88.48}\,{\fontsize{6.5}{7.5}\selectfont(\textcolor{red}{+0.54})} & \textbf{91.53}\,{\fontsize{6.5}{7.5}\selectfont(\textcolor{red}{+0.35})} & \textbf{77.02}\,{\fontsize{6.5}{7.5}\selectfont(\textcolor{red}{+0.88})} & \textbf{86.81}\,{\fontsize{6.5}{7.5}\selectfont(\textcolor{red}{+0.71})} \\
\bottomrule
\end{tabular}
\endgroup
\par\vspace{8pt}
\begingroup
\normalsize
\makeatletter\def\@captype{figure}\makeatother
\centering
	\includegraphics[width=\textwidth]{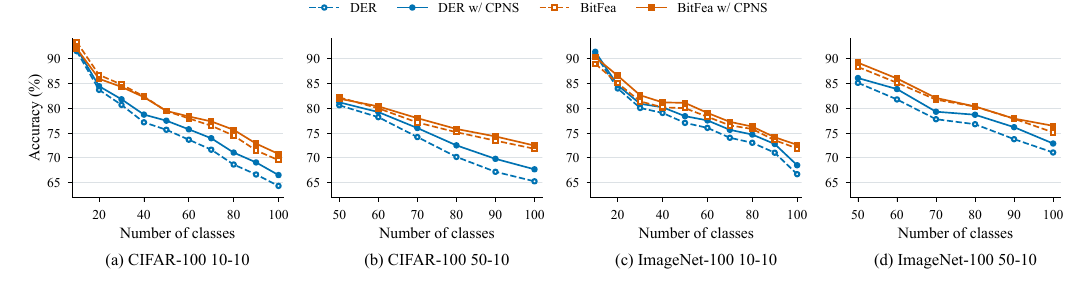}
	\caption{Accuracy curves for DER and BitFea with and without CPNS.}
	\label{zhexian}
\endgroup
\end{table*}

\subsection{Experimental Setup}

\textbf{Datasets.} We evaluate the method on three standard datasets, namely CIFAR-100~\cite{krizhevsky2009learning} and ImageNet-100/1000~\cite{deng2009imagenet}. We further use four fine-grained benchmarks, namely CUB200~\cite{wah2011caltech}, Birds525~\cite{bird525}, Flower102~\cite{flower102}, and Food101~\cite{food101}, to assess discriminative performance under high visual similarity.

\noindent\textbf{Data Split.} The experimental settings are denoted as $B$-$I$, where $B$ represents the number of classes in the initial base task, and $I$ denotes the number of new classes introduced in each subsequent incremental task. Under Equal Split (e.g., 10-10)~\cite{rebuffi2017icarl}, the 100 classes are divided equally. The model is first trained on 10 base classes, followed by 9 incremental steps, each containing 10 new classes. Under Half Split (e.g., 50-10)~\cite{hou2019learning,yu2020semantic}, the model starts with a larger base task containing 50 classes (half of the dataset). The remaining 50 classes are then learned sequentially in batches of 10 classes per step. For datasets where the total number of classes is not divisible by the step size (e.g., Food101 with 101 classes), we follow the convention of omitting the remaining classes to ensure each incremental task has an identical number of new categories.

\textbf{Model settings.}
We use the official implementations of DER~\cite{yan2021dynamically}, FOSTER~\cite{wang2022foster}, BEEF~\cite{wang2023beef}, and BitFea~\cite{zheng2026bi}. These experiments use 200 epochs for the base task, 170 epochs per incremental task, and a batch size of 128. The replay buffer contains 20,000 exemplars for ImageNet-1000 and 2,000 for the other datasets. Experiments run on NVIDIA RTX 3090 GPUs. Classification results are averaged over five runs.

\textbf{CPNS settings.}
The two-stage schedule uses $E_1=100$ and $\lambda=0.5$ within the original training budget. Except in the hyperparameter analysis, we set $K=2$ and $\tau_{\mathrm{intra}}=\tau_{\mathrm{inter}}=1$. The counterfactual generator in Eqs.~\eqref{Double-counterfactual1} and~\eqref{Double-counterfactual2} uses $T=3$ gradient steps with step size $\eta_{\mathrm{cf}}=0.05$. Class-conditional reference features are refreshed once per epoch.

\textbf{Pre-trained models.}
We also evaluate CPrompt~\cite{Gao_2024_CVPR} and CL-LoRA~\cite{He_2025_CVPR} with frozen pre-trained ViT-B/16 backbones and no replay. CPNS-intra minimizes $\mathcal L_{\mathrm{base}}+\widehat{\widetilde R}_{\mathrm{intra}}$ on prompted or LoRA-adapted features, using logits restricted to $\mathcal C_t$. Reference sets contain current-task training features and are discarded after each task. We retain each baseline's checkpoints, training budget, and inference rules, and update its trainable parameters.

\textbf{Evaluation metrics.}
We report the test accuracy after the final incremental task (Last) and the mean test accuracy across all stages, including the base task (Avg).

\subsection{Performance Results}

Table~\ref{tab:results} reports results on CIFAR-100~\cite{krizhevsky2009learning}, ImageNet-100, and ImageNet-1000~\cite{deng2009imagenet}. CPNS improves both Last and Avg accuracy for DER, FOSTER, BEEF, and BitFea across all five settings, with gains of 0.65--2.42 and 0.24--1.80 percentage points, respectively. The improvements hold under both the 10--10 and 50--10 splits on CIFAR-100 and ImageNet-100, and extend to the 1,000-class setting. On ImageNet-1000 100--100, CPNS raises BitFea's Last accuracy from 62.31\% to 63.21\% and Avg accuracy from 68.98\% to 70.42\%. Fig.~\ref{zhexian} further shows that the accuracy gains for DER persist across successive incremental stages on CIFAR-100 and ImageNet-100, alongside the BitFea trajectories with and without CPNS. These results demonstrate improvements in both final classification performance and accuracy over the incremental sequence.

For the rehearsal-free CPrompt and CL-LoRA baselines, CPNS-intra improves both metrics across all five settings. Last and Avg accuracy increase by 0.38--1.08 and 0.35--1.16 percentage points, respectively. On ImageNet-1000, CPrompt gains 1.08 points in Last accuracy and 1.16 points in Avg accuracy. These results support the use of the intra-task objective with both prompt tuning and low-rank adaptation on frozen pre-trained backbones.


\FloatBarrier

\begin{table*}[!t]
\centering
\caption{Performance results on various fine-grained datasets.}
\label{tab:cub200}
\fontsize{8}{9.5}\selectfont
\setlength{\tabcolsep}{1.5pt}
\renewcommand{\arraystretch}{1.12}
\setlength{\aboverulesep}{0.4ex}
\setlength{\belowrulesep}{0.55ex}
\begin{tabular}{
>{\raggedright\arraybackslash}p{\dimexpr0.12\textwidth-2\tabcolsep\relax}
*{8}{>{\centering\arraybackslash}p{\dimexpr0.11\textwidth-2\tabcolsep\relax}}}
\toprule
\multirow{2}{*}{Methods} & \multicolumn{2}{c}{CUB200 100-20} & \multicolumn{2}{c}{Birds525 100-50} & \multicolumn{2}{c}{Flower102 10-10} & \multicolumn{2}{c}{Food101 10-10} \\
\cmidrule(lr){2-3} \cmidrule(lr){4-5} \cmidrule(lr){6-7} \cmidrule(lr){8-9}
 & Last & Avg & Last & Avg & Last & Avg & Last & Avg \\
\midrule
DER & 52.88 & 53.07 & 77.40 & 85.21 & 48.70 & 55.94 & 61.10 & 72.05 \\
\rowcolor{lightgray}
w/ CPNS & \textbf{56.86}\,{\fontsize{6.5}{7.5}\selectfont(\textcolor{red}{+3.98})} & \textbf{55.49}\,{\fontsize{6.5}{7.5}\selectfont(\textcolor{red}{+2.42})} & \textbf{79.84}\,{\fontsize{6.5}{7.5}\selectfont(\textcolor{red}{+2.44})} & \textbf{87.99}\,{\fontsize{6.5}{7.5}\selectfont(\textcolor{red}{+2.78})} & \textbf{51.31}\,{\fontsize{6.5}{7.5}\selectfont(\textcolor{red}{+2.61})} & \textbf{57.92}\,{\fontsize{6.5}{7.5}\selectfont(\textcolor{red}{+1.98})} & \textbf{63.58}\,{\fontsize{6.5}{7.5}\selectfont(\textcolor{red}{+2.48})} & \textbf{74.21}\,{\fontsize{6.5}{7.5}\selectfont(\textcolor{red}{+2.16})} \\
\midrule

FOSTER & 53.48 & 54.01 & 78.06 & 85.36 & 49.64 & 57.77 & 61.96 & 72.54 \\
\rowcolor{lightgray}
w/ CPNS & \textbf{56.65}\,{\fontsize{6.5}{7.5}\selectfont(\textcolor{red}{+3.17})} & \textbf{56.82}\,{\fontsize{6.5}{7.5}\selectfont(\textcolor{red}{+2.81})} & \textbf{80.61}\,{\fontsize{6.5}{7.5}\selectfont(\textcolor{red}{+2.55})} & \textbf{88.02}\,{\fontsize{6.5}{7.5}\selectfont(\textcolor{red}{+2.66})} & \textbf{52.59}\,{\fontsize{6.5}{7.5}\selectfont(\textcolor{red}{+2.95})} & \textbf{60.71}\,{\fontsize{6.5}{7.5}\selectfont(\textcolor{red}{+2.94})} & \textbf{64.63}\,{\fontsize{6.5}{7.5}\selectfont(\textcolor{red}{+2.67})} & \textbf{74.79}\,{\fontsize{6.5}{7.5}\selectfont(\textcolor{red}{+2.25})} \\
\midrule

BEEF  & 51.30 & 51.98 & 76.98 & 85.64 & 48.67 & 56.04 & 60.99 & 71.76 \\
\rowcolor{lightgray}
w/ CPNS & \textbf{54.62}\,{\fontsize{6.5}{7.5}\selectfont(\textcolor{red}{+3.32})} & \textbf{55.09}\,{\fontsize{6.5}{7.5}\selectfont(\textcolor{red}{+3.11})} & \textbf{79.96}\,{\fontsize{6.5}{7.5}\selectfont(\textcolor{red}{+2.98})} & \textbf{87.83}\,{\fontsize{6.5}{7.5}\selectfont(\textcolor{red}{+2.19})} & \textbf{51.69}\,{\fontsize{6.5}{7.5}\selectfont(\textcolor{red}{+3.02})} & \textbf{58.14}\,{\fontsize{6.5}{7.5}\selectfont(\textcolor{red}{+2.10})} & \textbf{63.68}\,{\fontsize{6.5}{7.5}\selectfont(\textcolor{red}{+2.69})} & \textbf{73.76}\,{\fontsize{6.5}{7.5}\selectfont(\textcolor{red}{+2.00})} \\
\midrule


BitFea & 55.73 & 57.34 & 79.78 & 88.41 & 53.64 & 58.42 & 63.70 & 74.91 \\
\rowcolor{lightgray}
w/ CPNS & \textbf{56.44}\,{\fontsize{6.5}{7.5}\selectfont(\textcolor{red}{+0.71})} & \textbf{58.61}\,{\fontsize{6.5}{7.5}\selectfont(\textcolor{red}{+1.27})} & \textbf{81.39}\,{\fontsize{6.5}{7.5}\selectfont(\textcolor{red}{+1.61})} & \textbf{89.80}\,{\fontsize{6.5}{7.5}\selectfont(\textcolor{red}{+1.39})} & \textbf{55.59}\,{\fontsize{6.5}{7.5}\selectfont(\textcolor{red}{+1.95})} & \textbf{60.18}\,{\fontsize{6.5}{7.5}\selectfont(\textcolor{red}{+1.76})} & \textbf{65.47}\,{\fontsize{6.5}{7.5}\selectfont(\textcolor{red}{+1.77})} & \textbf{75.91}\,{\fontsize{6.5}{7.5}\selectfont(\textcolor{red}{+1.00})} \\
\midrule

CPrompt & 80.24 & 86.11 & 86.24 & 87.88 & 89.22 & 89.63 & 84.13 & 89.98 \\
\rowcolor{lightgray}
w/ CPNS-intra & \textbf{81.31}\,{\fontsize{6.5}{7.5}\selectfont(\textcolor{red}{+1.07})} & \textbf{86.98}\,{\fontsize{6.5}{7.5}\selectfont(\textcolor{red}{+0.87})} & \textbf{86.80}\,{\fontsize{6.5}{7.5}\selectfont(\textcolor{red}{+0.56})} & \textbf{88.54}\,{\fontsize{6.5}{7.5}\selectfont(\textcolor{red}{+0.66})} & \textbf{90.26}\,{\fontsize{6.5}{7.5}\selectfont(\textcolor{red}{+1.04})} & \textbf{90.51}\,{\fontsize{6.5}{7.5}\selectfont(\textcolor{red}{+0.88})} & \textbf{85.32}\,{\fontsize{6.5}{7.5}\selectfont(\textcolor{red}{+1.19})} & \textbf{90.45}\,{\fontsize{6.5}{7.5}\selectfont(\textcolor{red}{+0.47})} \\
\midrule

CL-LoRA & 74.59 & 81.70 & 86.24 & 87.76 & 87.90 & 89.20 & 82.47 & 89.18 \\
\rowcolor{lightgray}
w/ CPNS-intra & \textbf{75.89}\,{\fontsize{6.5}{7.5}\selectfont(\textcolor{red}{+1.30})} & \textbf{82.43}\,{\fontsize{6.5}{7.5}\selectfont(\textcolor{red}{+0.73})} & \textbf{87.10}\,{\fontsize{6.5}{7.5}\selectfont(\textcolor{red}{+0.86})} & \textbf{88.52}\,{\fontsize{6.5}{7.5}\selectfont(\textcolor{red}{+0.76})} & \textbf{88.64}\,{\fontsize{6.5}{7.5}\selectfont(\textcolor{red}{+0.74})} & \textbf{90.25}\,{\fontsize{6.5}{7.5}\selectfont(\textcolor{red}{+1.05})} & \textbf{83.55}\,{\fontsize{6.5}{7.5}\selectfont(\textcolor{red}{+1.08})} & \textbf{90.31}\,{\fontsize{6.5}{7.5}\selectfont(\textcolor{red}{+1.13})} \\
\bottomrule
\end{tabular}
\end{table*}

\subsection{Performance on Fine-Grained Datasets} 
Fine-grained benchmarks evaluate class-incremental discrimination among visually similar categories. Table~\ref{tab:cub200} reports results on CUB200, Birds525, Flower102, and Food101. CPNS improves both Last and Avg accuracy for all four expansion-based baselines on every dataset. Averaged over the 16 baseline--dataset combinations, the gains are 2.56 points in Last accuracy and 2.18 points in Avg accuracy. On CUB200, Last accuracy increases by 0.71--3.98 points across the four baselines, with DER improving from 52.88\% to 56.86\%. The gains also extend beyond bird recognition, with FOSTER's Avg accuracy increasing by 2.94 points on Flower102 and 2.25 points on Food101. These improvements support the effectiveness of CPNS for incremental classification among visually similar classes across different expansion methods.

CPNS-intra also improves CPrompt and CL-LoRA on all four fine-grained datasets, with Last and Avg gains of 0.56--1.30 and 0.47--1.13 points, respectively. On CUB200, their Last accuracy increases by 1.07 and 1.30 points, respectively. Together with the expansion-based results, these findings show that counterfactual regularization benefits fine-grained recognition when learning new feature branches and when adapting pre-trained representations.

\begin{table*}[!t]
\centering
\begin{minipage}[t]{\columnwidth}
\vspace{0pt}
\centering
\caption{Background intervention on CUB200 100-20.
Align. and Conflict denote aligned and conflicting backgrounds.
}
\label{tab:background_intervention}
\footnotesize
\setlength{\tabcolsep}{2pt}
\renewcommand{\arraystretch}{1.12}
\begin{tabular}{
>{\raggedright\arraybackslash}p{\dimexpr0.24\columnwidth-2\tabcolsep\relax}
*{3}{>{\centering\arraybackslash}p{\dimexpr0.18\columnwidth-2\tabcolsep\relax}}
>{\centering\arraybackslash}p{\dimexpr0.22\columnwidth-2\tabcolsep\relax}}
\toprule
\multirow{2}{*}{Method}
& \multicolumn{2}{c}{Last Accuracy (\%) $\uparrow$}
& \multirow{2}{*}{Drop $\downarrow$}
& \multirow{2}{*}{\shortstack{Old$\rightarrow$New\\ (\%) $\downarrow$}} \\
\cmidrule(lr){2-3}
& Align. & Conflict & & \\
\midrule
DER
& 57.50 & 48.63 & 8.87 & 18.46 \\

\rowcolor{lightgray}  w/ CPNS
& \textbf{59.22} & \textbf{53.49} & \textbf{5.73} & \textbf{15.08} \\
\midrule
FOSTER
& 57.64 & 49.51 & 8.13 & 17.81 \\
\rowcolor{lightgray} w/ CPNS
& \textbf{59.72} & \textbf{54.19} & \textbf{5.53} & \textbf{15.31} \\
\midrule
BEEF
& 56.42 & 47.21 & 9.21 & 19.38 \\
\rowcolor{lightgray} w/ CPNS
& \textbf{58.17} & \textbf{51.89} & \textbf{6.28} & \textbf{16.01} \\
\midrule
BitFea
& 58.41 & 51.54 & 6.87 & 17.29 \\

\rowcolor{lightgray}  w/ CPNS
& \textbf{60.12} & \textbf{55.52} & \textbf{4.60} & \textbf{14.88} \\
\midrule
CPrompt
& 81.92 & 78.93 & 2.99 & 4.46 \\
\rowcolor{lightgray} w/ CPNS-intra
& \textbf{82.80} & \textbf{81.17} & \textbf{1.63} & \textbf{3.05} \\
\midrule
CL-LoRA
& 76.71 & 72.92 & 3.79 & 6.45 \\
\rowcolor{lightgray} w/ CPNS-intra
& \textbf{77.76} & \textbf{75.60} & \textbf{2.16} & \textbf{4.76} \\
\bottomrule
\end{tabular}
\end{minipage}\hfill
\begin{minipage}[t]{\columnwidth}
\vspace{0pt}
\centering
    \caption{Last accuracy (\%) under top-$k$ semantic-part masking on CUB200 100-20.}
    \label{tab:masking_results}
    \footnotesize
    \setlength{\tabcolsep}{2pt}
    \renewcommand{\arraystretch}{1.12}
    \begin{tabular}{
    >{\raggedright\arraybackslash}p{\dimexpr0.24\columnwidth-2\tabcolsep\relax}
    *{5}{>{\centering\arraybackslash}p{\dimexpr0.116\columnwidth-2\tabcolsep\relax}}
    >{\centering\arraybackslash}p{\dimexpr0.18\columnwidth-2\tabcolsep\relax}}
        \toprule
        \multirow{2}{*}{Method} & \multicolumn{5}{c}{Top-$k$ Parts Masked} & \multirow{2}{*}{Avg. Drop $\downarrow$} \\
        \cmidrule(lr){2-6}
        & 0  & 1 & 2 & 3 & 5 & \\
        \midrule
        DER & 52.88 & 40.52 & 32.15 & 25.33 & 17.65 & 7.05 \\
        \rowcolor{lightgray} w/ CPNS & \textbf{56.86} & \textbf{47.75} & \textbf{40.23} & \textbf{34.57} & \textbf{26.19} & \textbf{6.13} \\
        \midrule
        FOSTER & 53.48 & 41.39 & 33.44 & 27.53 & 18.64 & 6.97 \\
        \rowcolor{lightgray} w/ CPNS & \textbf{56.65} & \textbf{46.93} & \textbf{34.81} & \textbf{34.38} & \textbf{27.41} & \textbf{5.85} \\
        \midrule
        BEEF & 51.30 & 38.18 & 30.45 & 23.85 & 16.63 & 6.93 \\
        \rowcolor{lightgray} w/ CPNS & \textbf{54.62} & \textbf{43.68} & \textbf{36.91} & \textbf{30.83} & \textbf{23.56} & \textbf{6.21} \\
        \midrule

        BitFea & 55.73 & 43.12 & 35.41 & 28.85 & 20.34 & 7.08 \\
        \rowcolor{lightgray} w/ CPNS & \textbf{56.44} & \textbf{47.51} & \textbf{41.06} & \textbf{35.55} & \textbf{27.87} & \textbf{5.71} \\
        \midrule
CPrompt & 80.24 & 70.06 & 63.00 & 56.84 & 49.46 & 6.16 \\
        \rowcolor{lightgray} w/ CPNS-intra & \textbf{81.31} & \textbf{72.81} & \textbf{65.49} & \textbf{59.09} & \textbf{52.46} & \textbf{5.77} \\
        \midrule
        CL-LoRA & 74.59 & 63.82 & 56.35 & 49.85 & 42.04 & 6.51 \\
        \rowcolor{lightgray} w/ CPNS-intra & \textbf{75.89} & \textbf{66.01} & \textbf{58.41} & \textbf{52.78} & \textbf{45.84} & \textbf{6.01} \\
        \bottomrule
    \end{tabular}
\end{minipage}
\end{table*}

\subsection{Robustness to Background Changes and Part Masking}

\textbf{Background Dependence on CUB-200.}
We evaluate background dependence on CUB-200 under the 100-20 setting.
We use segmentation masks to replace image backgrounds while preserving
bird foregrounds and class labels.
Each class is randomly assigned to one of three background types,
namely forest, grassland, or water.
Classes are distributed approximately equally across the three types
within each task.
During training, 90\% of each class's images use its assigned background
type. The remaining 10\% use the other two types equally. At test time, each bird is placed on backgrounds from all three types
without changing its foreground pixels, position, or scale.
We evaluate two conditions.
\textbf{Aligned} uses the assigned background type.
\textbf{Conflicting} averages accuracy across the two unassigned types.

We evaluate DER, FOSTER, BEEF, and BitFea with and without full CPNS.
CPrompt and CL-LoRA are evaluated with and without CPNS-intra.
Within each baseline, both variants use the same class order, background
assignments, test images, and training budget.
We report final-stage accuracy (Last) under each condition and its
drop from aligned to conflicting backgrounds.
We also report the percentage of old-class samples misclassified as
classes from the current task (Old$\rightarrow$New) under conflicting
backgrounds.

The background intervention tests whether CPNS learns class evidence beyond spurious background associations. Table~\ref{tab:background_intervention} shows that CPNS improves conflicting-background accuracy by 3.98--4.86 percentage points across the four expansion-based baselines and reduces the aligned-to-conflicting drop by 2.27--3.14 points. CPNS-intra improves conflicting-background accuracy for CPrompt and CL-LoRA by 2.24 and 2.68 points and reduces their drops by 1.36 and 1.63 points. Old$\rightarrow$New errors also decrease by 2.41--3.38 points for the expansion-based baselines. With bird foregrounds fixed, the smaller drops indicate reduced background dependence. These results support the CPNS objective of learning class-discriminative object evidence through factual--counterfactual comparisons, improving robustness when background--class correlations change.

\textbf{Semantic Part Removal.}
For each test image, we combine Grad-CAM responses with the CUB200 part annotations to rank semantic parts~\cite{behzadi2023protocol}. We progressively mask the top $k\in\{1,2,3,5\}$ parts and measure Last accuracy, denoted by $\mathrm{Acc}(k)$. The unmasked accuracy is $\mathrm{Acc}(0)$. We report the accuracy at each masking level and the mean drop per removed part, $[\mathrm{Acc}(0)-\mathrm{Acc}(5)]/5$.

The masking experiment tests whether CPNS retains discriminative evidence beyond a few dominant object parts. In Table~\ref{tab:masking_results}, removing the top-ranked part lowers DER accuracy by 12.36 percentage points, compared with 9.11 points with CPNS. After five parts are masked, CPNS improves accuracy by 6.93--8.77 points across the four expansion-based baselines and reduces the mean drop per part by 0.72--1.37 points. CPNS-intra also reduces the mean drop for CPrompt and CL-LoRA by 0.39 and 0.50 points. The smaller losses relative to each model's unmasked accuracy indicate that useful class evidence remains outside the removed regions. This supports the CPNS objective of learning a broader set of discriminative attributes through counterfactual class comparisons.


\FloatBarrier

\subsection{Counterfactual Generation Quality}

CPNS evaluates prediction necessity using feature interventions within and across tasks. We assess target-class prediction and feature displacement on current-class test samples correctly classified by both factual heads. Evaluation samples and targets, generation budgets, and metric estimators are specified in \textbf{Appendix C}.

\begin{figure}[!htbp]
	\centering
	\includegraphics[width=0.95\columnwidth]{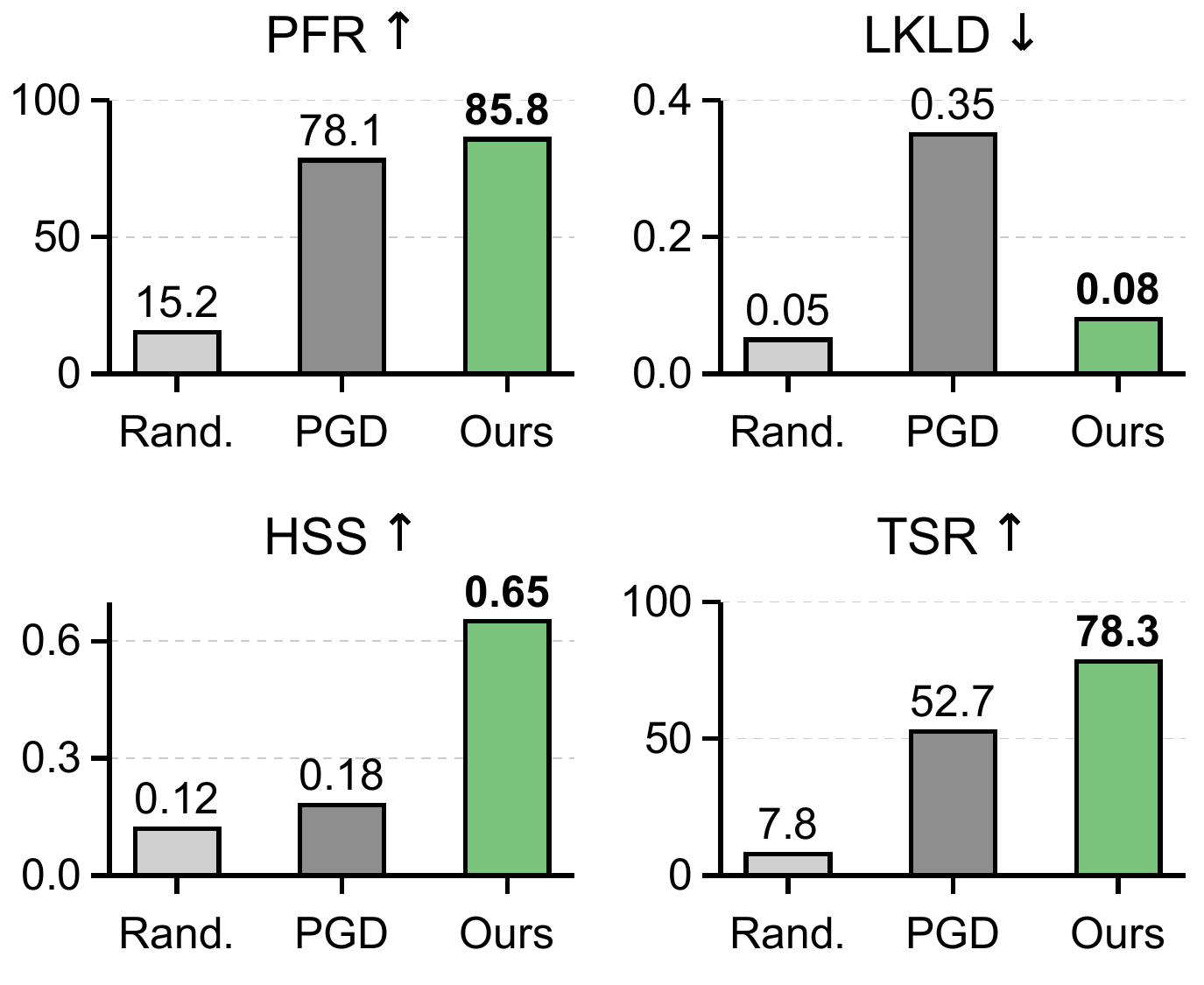}
	\caption{Counterfactual generation on CIFAR-100 with random perturbation (Rand.), PGD, and our generator (Ours).}
	\label{Counterfactual}
\end{figure}

\textbf{Generation quality.}
We compare CPNS with random perturbation and targeted Projected Gradient Descent (PGD)~\cite{madry2018towards} on CIFAR-100 to assess validity, minimality, and plausibility. Prediction Flip Rate (PFR) records changes from the factual prediction, while Target-Class Success Rate (TSR) directly assesses validity by measuring attainment of the specified competing class~\cite{carlini2017robustness,mothilal2020diverse}. Latent Kullback--Leibler Divergence (LKLD) compares class-conditional diagonal Gaussian fits to factual and generated features, providing indirect evidence of minimality through distributional displacement. Historical Semantic Similarity (HSS) assesses semantic plausibility for current-to-old interventions through cosine similarity to the projected historical prototype of the target class. The projection is fitted on factual training and replay features and used only for evaluation.

Fig.~\ref{Counterfactual} shows that our generator achieves a PFR of $85.8\%$ and a TSR of $78.3\%$, compared with $78.1\%$ and $52.7\%$ for PGD and $15.2\%$ and $7.8\%$ for random perturbation. The $25.6$-percentage-point TSR gain over PGD indicates more effective attainment of the specified competing class. Our generator has an LKLD of $0.08$, compared with $0.35$ for PGD and $0.05$ for random perturbation. Thus, CPNS achieves higher target-class success with less distributional displacement than PGD. Random perturbation produces the smallest displacement but reaches the target class in only $7.8\%$ of cases, highlighting the need to assess displacement together with validity. Our generator also achieves an HSS of $0.65$, compared with $0.18$ for PGD and $0.12$ for random perturbation. This stronger alignment with historical target-class prototypes supports the semantic plausibility of the inter-task features.

Together, these results support the dual-scope generation design. The generator more effectively redirects predictions to specified competing classes while limiting feature distributional displacement and aligning inter-task features with historical target-class semantics. These findings provide empirical support for using the generated counterfactual contrasts in CPNS to assess feature necessity within and across tasks.

\begin{figure}[!htbp]
	\centering
	\includegraphics[width=0.48\textwidth]{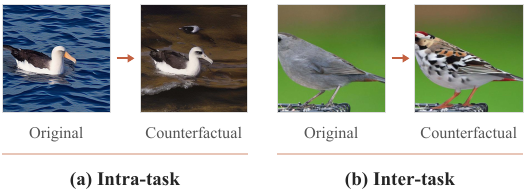}
	\caption{Image-space counterfactual references generated using Augustin et al.~\cite{augustin2022diffusion}. Panel (a) shows intra-task edits to the background and local head appearance, and panel (b) shows inter-task edits to plumage and leg coloration.}
	\label{counterfactuals}
\end{figure}

\textbf{Counterfactual generation comparison.} We assess the semantic consistency and computational cost of direct feature-space generation. Reference images are constructed using the diffusion-based method of Augustin et al.~\cite{augustin2022diffusion}. Fig.~\ref{counterfactuals} shows intra-task edits to the water background and local head appearance, and inter-task edits to plumage patterns and leg coloration while retaining a similar pose and background.

We compare CPNS with an empirical optimal-transport baseline~\cite{delara2024transport} using the same reference images and frozen current-task encoder. After $\ell_2$ normalization, squared maximum mean discrepancy (MMD$^2$) measures distributional agreement between image-derived reference features and generated features~\cite{gretton2012kernel}, providing evidence for plausibility. Computational overhead measures generation cost and is normalized to CPNS separately for each scope. \textbf{Appendix C-V} specifies the kernel, aggregation, and timing protocol.

\begin{figure}[!htbp]
	\centering
	\includegraphics[width=0.95\columnwidth]{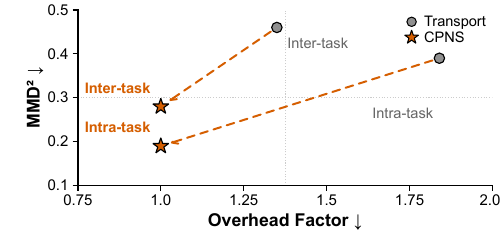}
	\caption{MMD$^2$ to image-derived counterfactual reference features and computational overhead for intra-task and inter-task generation. Lower is better on both axes. The overhead factor is normalized to CPNS within each scope.}
	\label{Counterfactual2}
\end{figure}

Fig.~\ref{Counterfactual2} shows MMD$^2$ values of $0.19$ and $0.28$ for CPNS in intra-task and inter-task generation, respectively, compared with $0.39$ and $0.46$ for transport. These correspond to discrepancy reductions of approximately $51.3\%$ and $39.1\%$. Transport incurs $1.84\times$ and $1.35\times$ the computational cost of CPNS in the two scopes. These results support CPNS's use of direct feature-space generation to construct counterfactual contrasts consistent with semantic image edits at lower cost than transport within and across tasks.

\FloatBarrier

\subsection{Further Analysis}

\setlength{\textfloatsep}{10pt plus 2pt minus 2pt}
\setlength{\floatsep}{8pt plus 2pt minus 2pt}
\setlength{\intextsep}{8pt plus 2pt minus 2pt}

\begin{table}[!htbp]
\centering
\caption{Ablation study on CIFAR-100 and ImageNet-100 in terms of average accuracy. The baseline is DER. $\mathrm{PNS}_{\mathrm{intra}}$ and $\mathrm{PNS}_{\mathrm{inter}}$ indicate inclusion of the corresponding surrogate losses. 2-Stage denotes delayed inter-task activation.}
\footnotesize
\setlength{\tabcolsep}{2pt}
\renewcommand{\arraystretch}{1.12}
\begin{tabular}{
*{3}{>{\centering\arraybackslash}p{\dimexpr0.16\columnwidth-2\tabcolsep\relax}}
*{4}{>{\centering\arraybackslash}p{\dimexpr0.13\columnwidth-2\tabcolsep\relax}}}
\toprule
\multicolumn{3}{c}{DER baseline}  & \multicolumn{2}{c}{CIFAR-100} & \multicolumn{2}{c}{ImageNet-100}   \\
\cmidrule(lr){1-3} \cmidrule(lr){4-5}  \cmidrule(lr){6-7}
$\text{PNS}_{\text{intra}}$ & $\text{PNS}_{\text{inter}}$ & 2-Stage  & 10-10 & 50-10  & 10-10 & 50-10\\
\midrule
 \xmark & \xmark & \xmark  & 75.36 & 72.60  & 77.18 & 77.71\\
\rowcolor{lightgray} \checkmark & \xmark & \xmark  & 76.31& 73.64  & 77.78 & 78.64\\
 \xmark & \checkmark & \xmark &  74.98 & 71.84 & 76.65 & 77.21\\
\rowcolor{lightgray}\xmark & \checkmark & \checkmark  & 75.28 & 73.59   & 77.43 & 78.51\\
 \checkmark & \checkmark & \xmark & 75.16 & 72.09 & 76.98 & 77.79 \\
\rowcolor{lightgray}\checkmark & \checkmark & \checkmark  & 77.08 & 74.40   & 78.46 & 79.49\\
\bottomrule
\end{tabular}
\label{strategies}
\end{table}

\noindent\textbf{Ablation study.} Table~\ref{strategies} evaluates the contributions of the two regularizers and the two-stage schedule using DER on CIFAR-100 and ImageNet-100. Intra-task regularization alone improves average accuracy by 0.60--1.04 percentage points over DER. Adding inter-task regularization with delayed activation yields a further 0.68--0.85 points. The full configuration performs best in all four settings, exceeding simultaneous activation by 1.48--2.31 points. These results support combining complementary intra-task and inter-task regularization through a schedule that first learns current-class distinctions and then introduces inter-task constraints.

\noindent\textbf{Effect of two-stage training.}
We compare joint and two-stage training with DER on CIFAR-100
10-10 at a fixed incremental task. Paired runs share the task-initial
checkpoint, data order, learning-rate schedule, and 170-epoch budget.
The inter-task loss starts at epoch~1 or~101, respectively,
with $\lambda=0.5$. At each epoch, fixed batches give
\mbox{$g_A=\nabla_{\theta_t}(\mathcal{L}_{\mathrm{base}}+
\widehat{\widetilde{R}}_{\mathrm{intra}})$} and
\mbox{$g_B=\nabla_{\theta_t}\widehat{\widetilde{R}}_{\mathrm{inter}}$}.
Their cosine similarity is
\begin{equation}
c =
\frac{\langle g_A,g_B\rangle}
{\lVert g_A\rVert_2\,\lVert g_B\rVert_2}.
\label{eq:gradient_alignment}
\end{equation}
Gradients follow Eq.~\eqref{eq:detached_perturbation}, with $c<0$ indicating conflicting updates. Before activation, $g_B$ is measured without parameter updates. Fig.~\ref{fig:gradient_alignment} reports epoch-wise batch averages as means and standard deviations across three paired runs. Early conflict under joint training and weaker opposition at delayed activation support the two-stage schedule for reducing interference between the two objectives.
\begin{figure}[!htbp]
\centering
\includegraphics[width=0.45\textwidth]{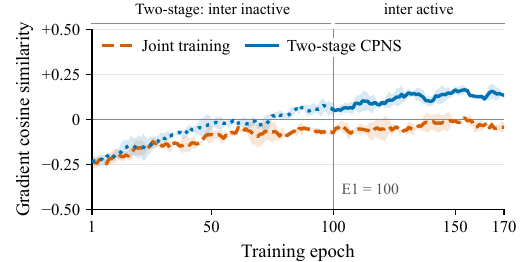}
\caption{Gradient alignment under joint and two-stage training
with DER on CIFAR-100 (10--10).}
\label{fig:gradient_alignment}
\end{figure}

\noindent\textbf{Comparison with Other Causal Methods.}
We assess CPNS's effectiveness for feature expansion by 
comparing it with two other causal methods, 
$C^3$~\cite{wangtowards} and CaSN~\cite{yang2023invariant},
 using DER on CIFAR-100 10-10 and CUB200 100-20.
  Both baselines are adapted to DER's feature dimensions, 
  expanding label space, and frozen historical branches. 
  All variants share the backbone, data splits, class orders, 
  replay budget, and training schedule.
   In Table~\ref{tab:prior_pns_comparison}, CPNS exceeds the
    best causal baseline in each column by 1.30 and 3.27 percentage points in final accuracy and 1.41 and 1.50 points in average accuracy on the two datasets, respectively. These gains support the effectiveness of CPNS's dual-scope PNS formulation for feature expansion, with a larger final-stage improvement on fine-grained CUB200.
\begin{table}[!htbp]
\centering
\caption{Comparison with PNS-based methods using DER.}
\label{tab:prior_pns_comparison}
\footnotesize
\setlength{\tabcolsep}{3pt}
\renewcommand{\arraystretch}{1.12}
\begin{tabular}{
>{\raggedright\arraybackslash}p{\dimexpr0.32\columnwidth-2\tabcolsep\relax}
*{4}{>{\centering\arraybackslash}p{\dimexpr0.17\columnwidth-2\tabcolsep\relax}}}
\toprule
\multirow{3}{*}{Method}
& \multicolumn{2}{c}{CIFAR-100}
& \multicolumn{2}{c}{CUB200} \\
& \multicolumn{2}{c}{10-10}
& \multicolumn{2}{c}{100-20} \\
\cmidrule(lr){2-3}\cmidrule(lr){4-5}
& Last $\uparrow$ & Avg $\uparrow$ & Last $\uparrow$ & Avg $\uparrow$ \\
\midrule
DER & 64.35 & 75.36 & 52.88 & 53.07 \\
DER + CaSN & 65.12 & 75.49 & 53.59 & 53.78 \\
DER + $C^3$ & 65.23 & 75.67 & 53.22 & 53.99 \\
\rowcolor{lightgray} DER + CPNS & 66.53 & 77.08 & 56.86 & 55.49 \\
\bottomrule
\end{tabular}
\end{table}

\noindent\textbf{Empirical assessment of monotonicity.}
We assess the monotonicity condition for CPNS identification in Theorem~\ref{theorem:indpns}. It requires correct classification with generated features to imply correct classification with factual features. Violations occur when feature replacement corrects factual errors. For scope $s$, let $\hat M_s$ denote the violation rate and $\hat\Delta_s$ the factual accuracy minus the generated accuracy. The identity $\widehat{\mathrm{PNS}}_s-\hat\Delta_s=\hat M_s$ (\textbf{Appendix A-II}) shows that the violation rate equals the empirical error of the accuracy-difference approximation.

\begin{figure}[!htbp]
\centering
\includegraphics[width=\columnwidth]{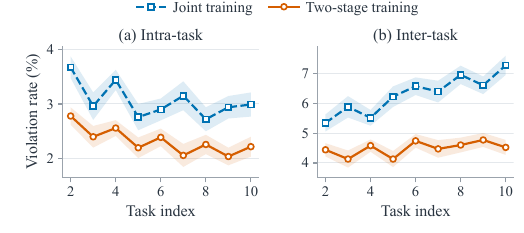}
\caption{Monotonicity-violation rates on CIFAR-100 10-10 with DER under
(a) intra-task and (b) inter-task interventions. Lines and bands show
means and one standard deviation.}
\label{fig:monotonicity_diagnostics}
\end{figure}

We evaluate DER on CIFAR-100 10-10 using all samples in each scope at every incremental checkpoint. Averaged over incremental tasks, intra-task and inter-task violation rates under two-stage training are $2.32\%$ and $4.49\%$, respectively, compared with $3.06\%$ and $6.30\%$ under joint training. Fig.~\ref{fig:monotonicity_diagnostics} shows lower rates with two-stage training in both scopes at all nine checkpoints. These results support using factual--generated accuracy differences for CPNS estimation, with two-stage training reducing approximation error in both scopes.

\noindent\textbf{Hyperparameter analysis.}
We study neighborhood size $K$, support threshold $\tau_{\mathrm{inter}}$, and inter-task weight $\lambda$ on CIFAR-100 10-10. Fig.~\ref{hunxiao} compares $K\in\{1,2,5,9\}$ and $\tau_{\mathrm{inter}}\in\{0.1,0.5,1,2\}$. Average accuracy ranges from $75.02\%$ to $77.08\%$, peaking at $K=2$ and $\tau_{\mathrm{inter}}=1$. At this threshold, increasing $K$ from 2 to 9 reduces accuracy by 2.06 percentage points. At $K=2$, $\tau_{\mathrm{inter}}=2$ yields $76.57\%$.

\begin{figure}[!htbp]
\centering
\includegraphics[width=0.895\linewidth]{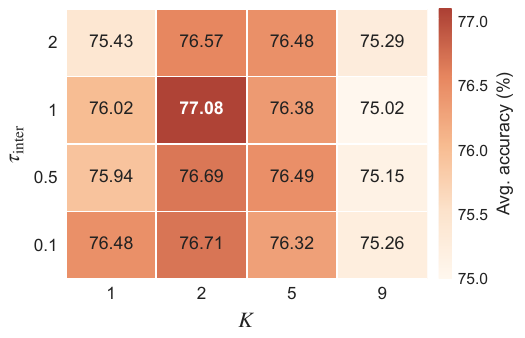}
\caption{Average accuracy for different neighborhood sizes $K$ and inter-task support thresholds $\tau_{\mathrm{inter}}$ on CIFAR-100 10-10.}
\label{hunxiao}
\end{figure}

\begin{figure*}[!t]
\centering
\begin{minipage}[t]{\columnwidth}
\centering
\includegraphics[width=\linewidth]{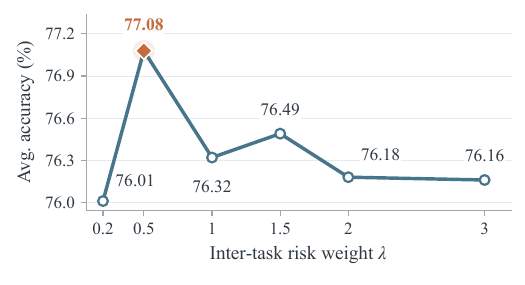}
\caption{Effect of the inter-task CPNS risk weight $\lambda$ on average accuracy on CIFAR-100 10-10.}
\label{canshu}
\end{minipage}\hfill
\begin{minipage}[t]{\columnwidth}
\centering
\includegraphics[width=0.9375\linewidth]{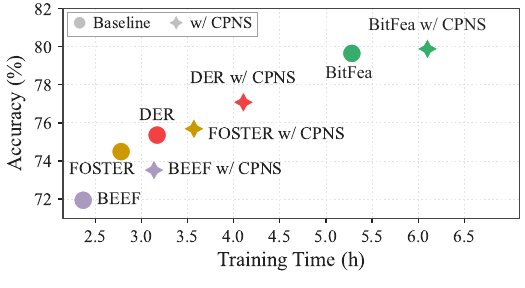}
\caption{Training time and accuracy of different methods in the CIFAR-100 10-10 scenario.}
\label{traintime}
\end{minipage}
\end{figure*}

\begin{figure*}[!t]
\centering
\includegraphics[width=\textwidth]{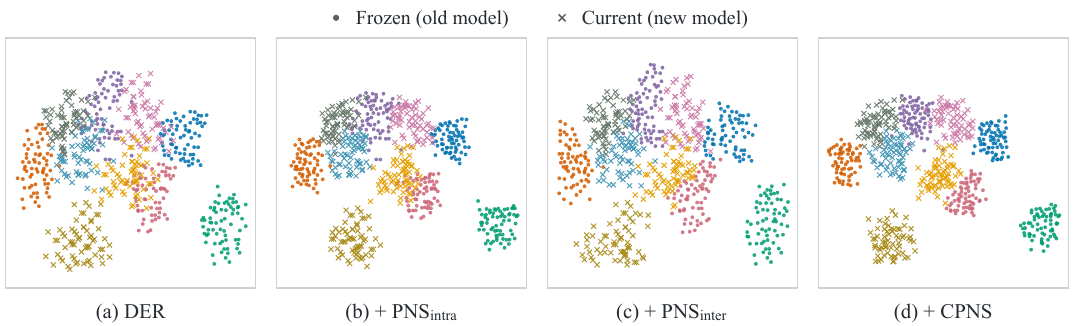}
\caption{t-SNE visualization on CUB200 for (a) DER, (b) DER with $\mathrm{PNS}_{\mathrm{intra}}$, (c) DER with $\mathrm{PNS}_{\mathrm{inter}}$, and (d) DER with CPNS. Colors denote classes. Circles and crosses indicate representations from the frozen old model and the current model, respectively.}
\label{fig:tsne}
\end{figure*}

Fig.~\ref{canshu} varies $\lambda\in\{0.2,0.5,1,1.5,2,3\}$ in Eq.~\eqref{eq:total_objective}. Accuracy peaks at $77.08\%$ for $\lambda=0.5$ and ranges from $76.01\%$ to $77.08\%$ for tested weights between 0.2 and 1.5. At $\lambda=2$ and $\lambda=3$, accuracy is $76.18\%$ and $76.16\%$, respectively. These results favor a small target-class neighborhood and moderate inter-task regularization.

\noindent\textbf{Model efficiency.}
Fig.~\ref{traintime} compares average accuracy and training time with and without CPNS for DER, FOSTER, BEEF, and BitFea on CIFAR-100 10-10. CPNS improves accuracy across all four baselines while adding approximately 0.8--0.9 training hours. Training time rises from about 3.2 to 4.1 hours for DER and from 5.3 to 6.1 hours for BitFea. The accuracy gains at approximately $26\%$ mean relative overhead support CPNS's integration into different expansion-based CIL methods.

\subsection{Visualization}
\label{Appendix:visualization}

\smallskip

\noindent\textbf{t-SNE.}
Fig.~\ref{fig:tsne} uses t-SNE~\cite{van2008visualizing} to visualize embeddings of ten CUB200 classes, five from the current task and five from the preceding task. Each class is represented by its corresponding task-specific module, with historical features frozen. DER exhibits dispersed within-class embeddings and overlap between current and historical classes (a). Intra-task regularization produces more compact current-class clusters (b), while inter-task regularization increases separation between current and historical representations (c). Full CPNS combines both patterns (d), supporting its dual-scope design for intra-task class discrimination and inter-task feature separability.

\begin{figure}[htbp]
    \centering
    \includegraphics[width=\columnwidth]{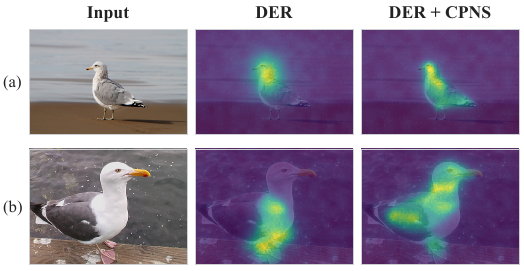}
    \caption{Grad-CAM visualizations on CUB200. Each row shows an input image and the corresponding overlays for DER and DER with CPNS.}
    \label{fig:gradcam}
\end{figure}

\noindent\textbf{Grad-CAM.} Fig.~\ref{fig:gradcam} compares Grad-CAM maps~\cite{selvaraju2017grad} for DER and DER with CPNS on CUB200~\cite{wah2011caltech}. In the illustrated examples, CPNS extends activation coverage across the birds\textquotesingle{} heads and bodies, while DER emphasizes more limited regions. This broader coverage supports CPNS's objective of learning discriminative cues from multiple object regions.


\endgroup

\section{Conclusion}

We propose CPNS, a dual-scope PNS-based regularizer for expansion-based CIL that promotes intra-task class discrimination and inter-task feature separability. Class-conditional KNN support penalties guide the generation of counterfactual features for CPNS risk estimation. Our theoretical analysis establishes the identifiability of CPNS, provides guarantees for intra-task class discrimination and inter-task feature separability, and derives a generalization bound for CPNS risk. Experiments on three standard and four fine-grained benchmarks demonstrate consistent improvements across multiple expansion-based CIL methods.

\noindent\textbf{Limitations.} Inter-task CPNS regularization relies on replay exemplars to provide historical class references. Rehearsal-free settings therefore
 use CPNS-intra with intra-task regularization alone. With a fixed memory budget, the 
 number of exemplars per old class decreases as more classes arrive, which may limit
  reference coverage of within-class variation. Counterfactual search and reference 
  feature updates also add training cost. Future work will explore compact reference representations for historical classes and more efficient counterfactual search.

\section{ACKNOWLEDGEMENT}
This work is supported by the National Natural Science Foundation of China (Nos. 62276218, U2468207 and 62506311), the Fundamental Research Funds for the Central Universities, China (No. 2682024ZTPY055).
\bibliography{references}
\bibliographystyle{IEEEtran}

\begingroup
\makeatletter
\def\@IEEEBIOphotowidth{0.85in}
\def\@IEEEBIOphotodepth{0.95in}
\def\@IEEEBIOhangwidth{0.99in}
\def\@IEEEBIOhangdepth{0.95in}
\def\@IEEEBIOskipN{1\baselineskip}
\xpatchcmd{\IEEEbiography}{\vskip \@IEEEBIOskipN plus 1fil minus 0\baselineskip}{\vskip \@IEEEBIOskipN plus 1fill}{}{\PackageError{cpns-layout}{Biography spacing patch failed}{}}
\raggedbottom

\makeatother

\begin{IEEEbiography}[{\includegraphics[width=0.85in,height=0.95in,clip,keepaspectratio]{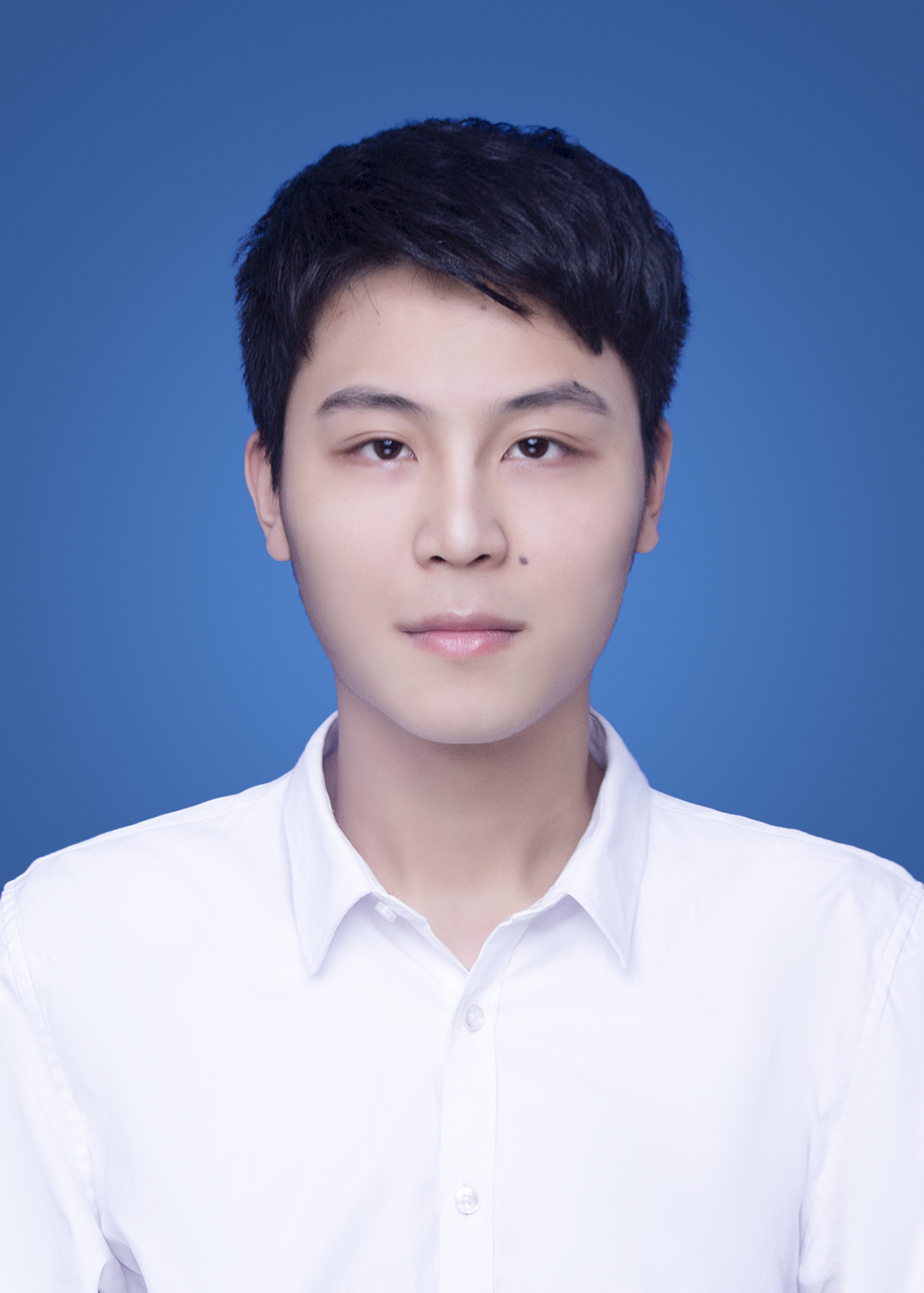}}]{Zhen Zhang}
is pursuing the Ph.D. degree in the School of Computing and Artificial Intelligence, Southwest Jiaotong University, Chengdu, China. His research interests include causal decoupling representation learning, causal inference, continual learning, and image quality assessment.
\end{IEEEbiography}

\begin{IEEEbiography}[{\includegraphics[width=0.85in,height=0.95in,clip,keepaspectratio]{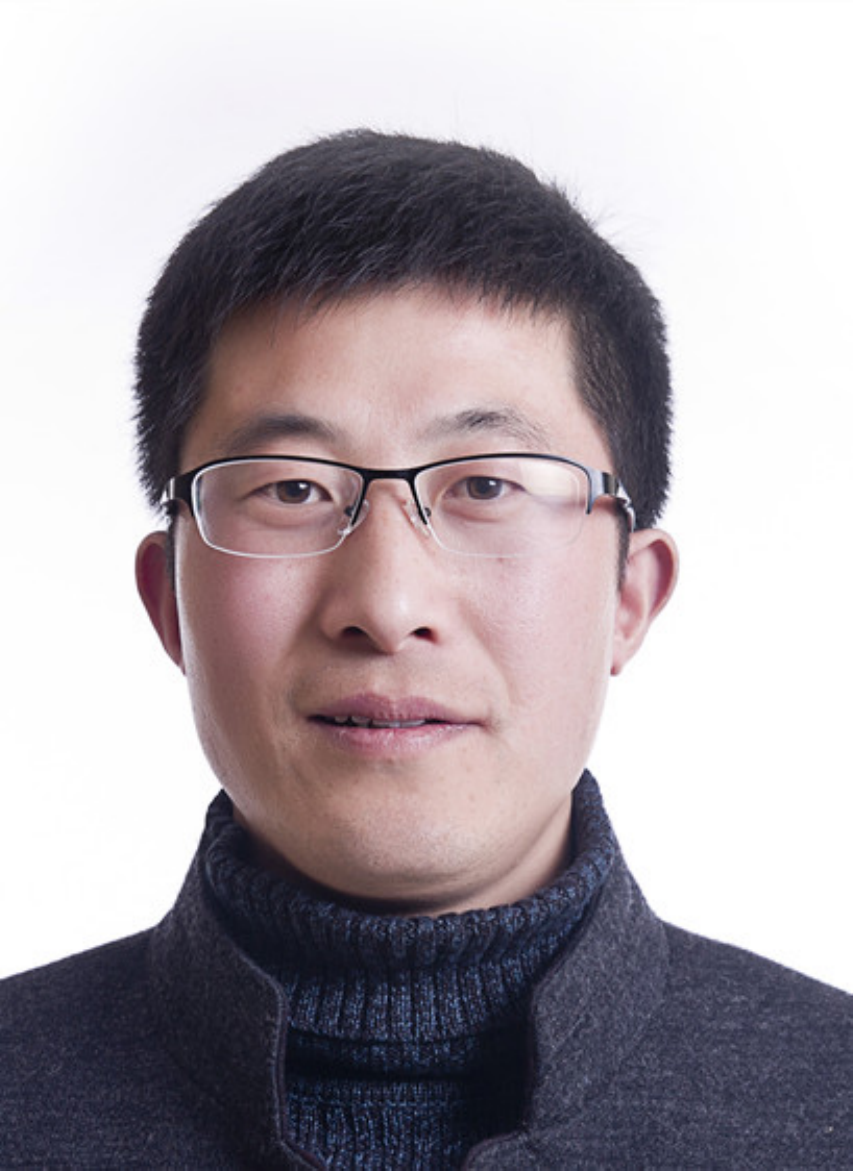}}]{Jielei Chu} (Senior Member, IEEE)
received the Ph.D. degree in computer science from Southwest Jiaotong University, Chengdu, China, in 2020. He serves as an Editorial Board Member for \emph{Scientific Reports}. He has published more than 40 papers in journals such as \emph{IEEE TPAMI}, \emph{IEEE TKDE}, \emph{IEEE TCYB}, and \emph{IEEE TMM}. His research interests include deep learning, semi-supervised learning, federated learning, and brain-inspired intelligence.
\end{IEEEbiography}
\begin{IEEEbiography}[{\includegraphics[width=0.85in,height=0.95in,clip,keepaspectratio]{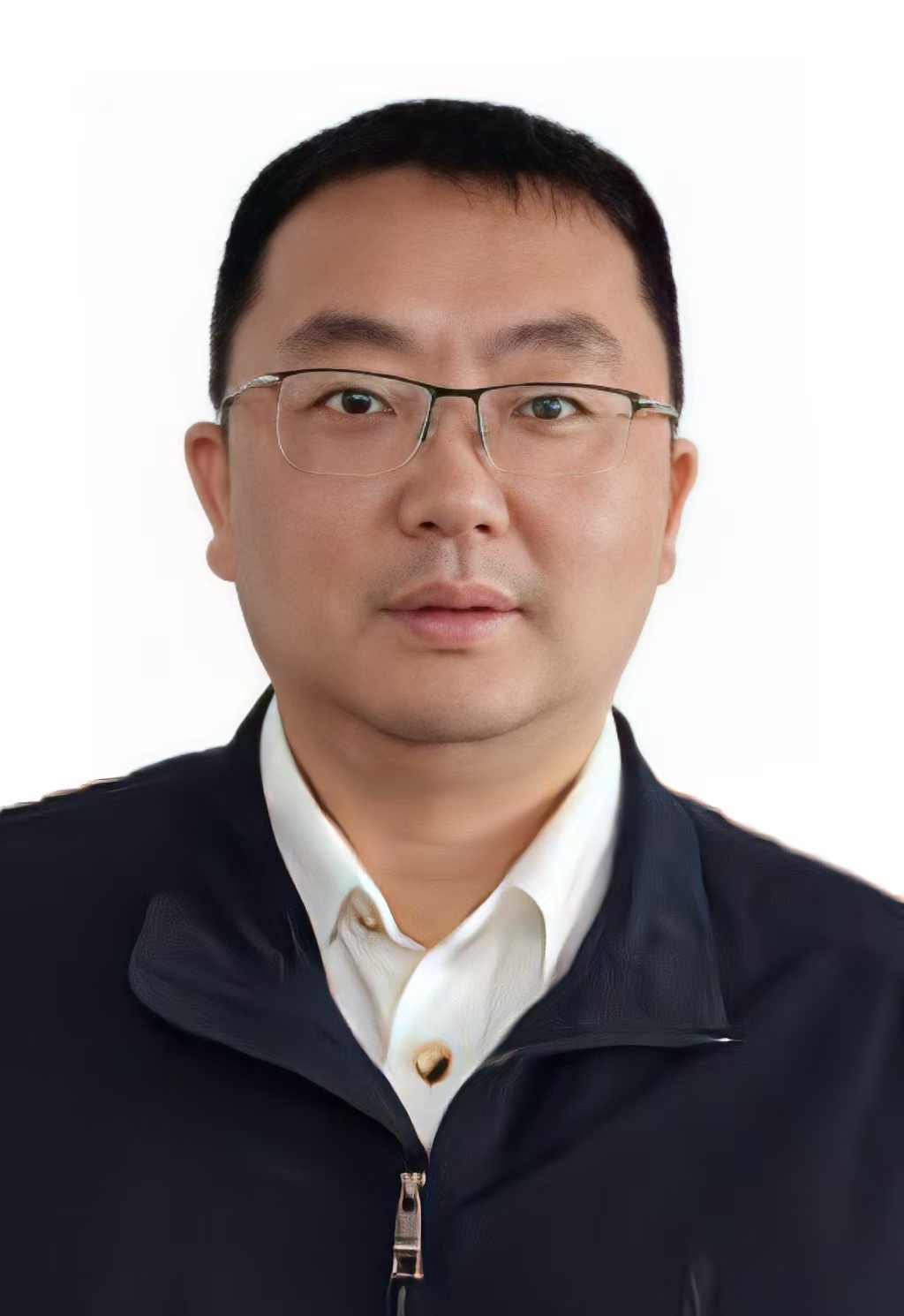}}]{Jiangtao Hu} received the M.S. degree in Public Health from Sichuan University, Chengdu, China, in 2005. He has published more than 30 papers in journals. His research interests include machine learning, data mining and  and public health testing.
\end{IEEEbiography}

\begin{IEEEbiography}[{\includegraphics[width=0.85in,height=0.95in,clip,keepaspectratio]{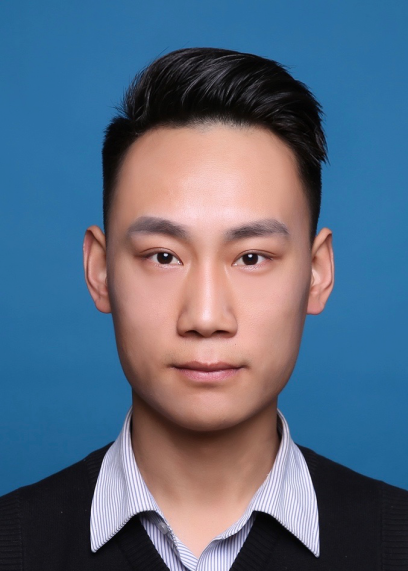}}]{Bin Liu} received the Ph.D. degree from School of Electronic Information and Communications, Huazhong University of Science and Technology (HUST), Wuhan, China, in 2023. Now he is with the School of Computing and Artificial Intelligence, Southwest Jiaotong University (SWJTU), Chengdu, China. He has published more than 10 scientific papers in prestigious international journals and conferences, including TPAMI, TKDE, ICDE. 
\end{IEEEbiography}

\begin{IEEEbiography}[{\includegraphics[width=0.85in,height=0.95in,clip,keepaspectratio]{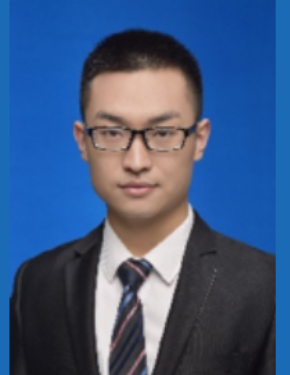}}]{Jie Wang} received the Ph.D. degree from the Southwest Jiaotong University, Chengdu, China, in 2024. He is currently an assistant researcher at School of Computing and Artificial Intelligence, Southwest Jiaotong University. His research interests include deep learning, multimodal learning and data mining.
\end{IEEEbiography}

\begin{IEEEbiography}[{\includegraphics[width=0.85in,height=0.95in,clip,keepaspectratio]{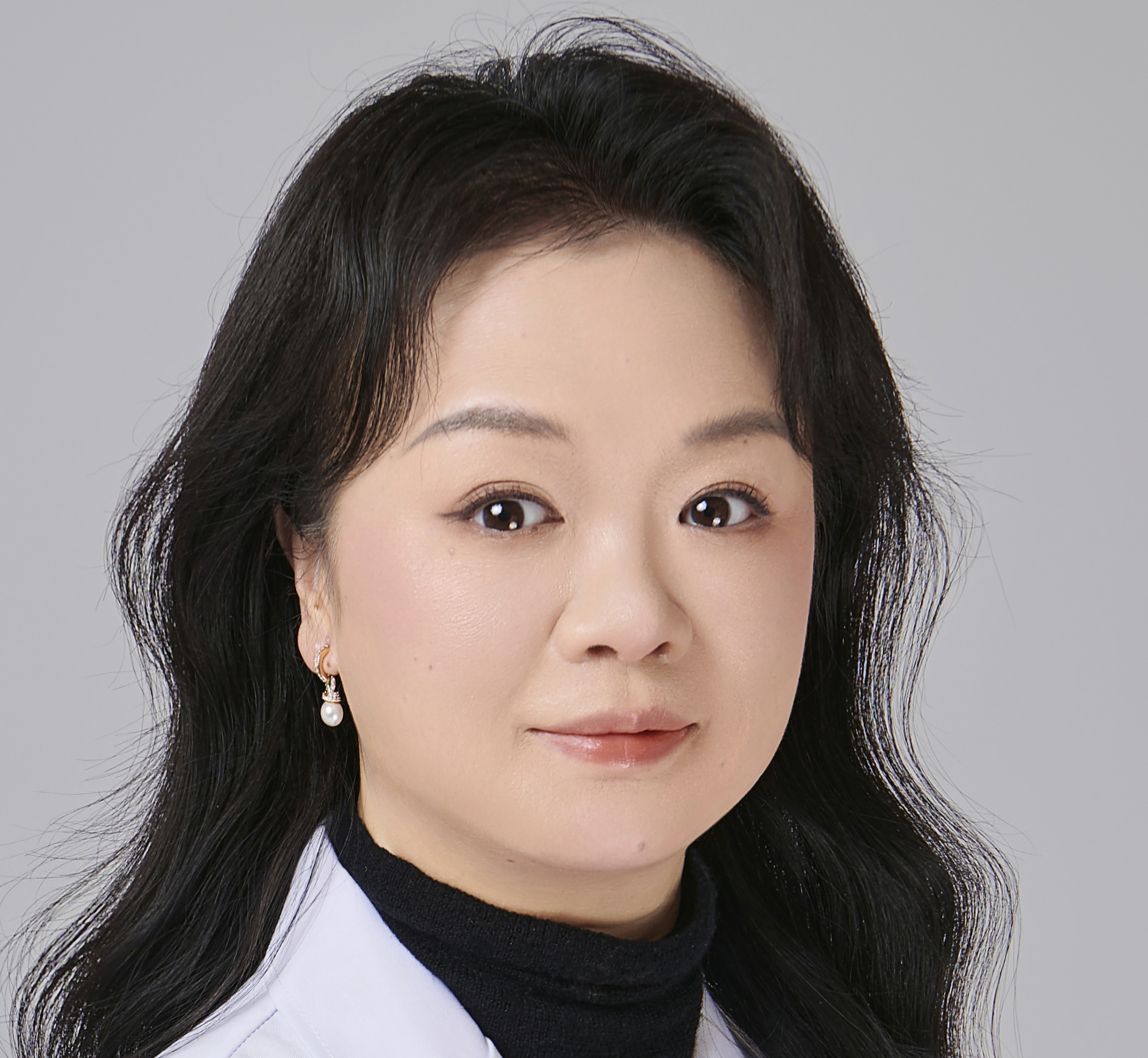}}]{Ya Liu} is with Chengdu University of Traditional Chinese Medicine, Chengdu, China. Her research interests include machine learning, artificial intelligence, and big data.
\end{IEEEbiography}
\begin{IEEEbiography}[{\includegraphics[width=0.85in,height=0.95in,clip,keepaspectratio]{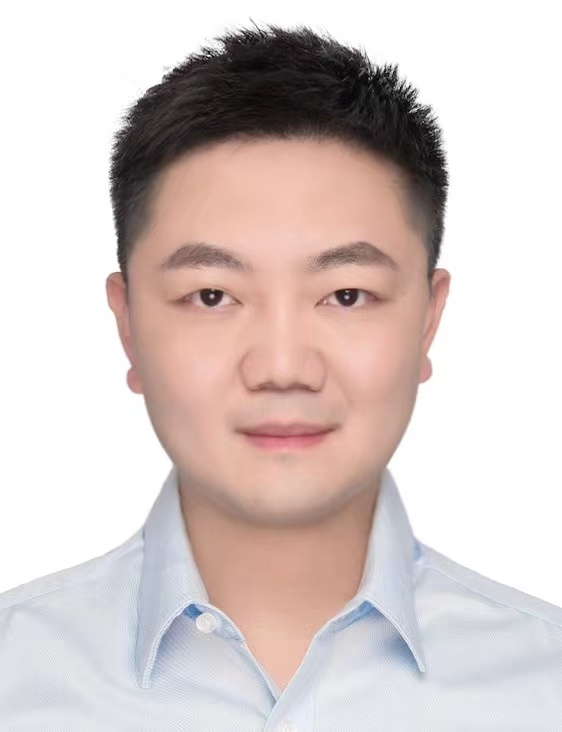}}]{Fengmao Lv} (Member, IEEE)
received the B.S. and Ph.D. degrees in computer science from the University of Electronic Science and Technology of China, Chengdu, China, in 2013 and 2018, respectively. He is currently an Associate Professor with Southwest Jiaotong University. He serves as an Associate Editor for IEEE Transactions on Circuits and Systems for Video Technology, Action Editor for Neural Networks, and Area Chair for ACM MM. His research interests include multimodal learning, open-world learning, and multimedia analysis.
\end{IEEEbiography}

\begin{IEEEbiography}[{\includegraphics[width=0.85in,height=0.95in,clip,keepaspectratio]{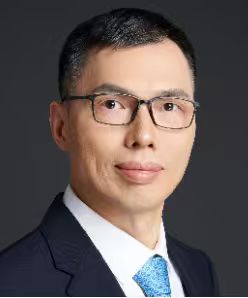}}]{Tianrui Li} (Senior Member, IEEE)
received the B.S., M.S., and Ph.D. degrees from Southwest Jiaotong University, Chengdu, China, in 1992, 1995, and 2002, respectively. He is currently a Professor and the Director of the Key Laboratory of Cloud Computing and Intelligent Techniques, Southwest Jiaotong University. He serves as Editor-in-Chief of \emph{Human-Centric Intelligent Systems}, Editor of \emph{Information Fusion}, and Associate Editor of \emph{ACM TIST}. He has authored or coauthored more than 500 papers in journals and conferences such as \emph{IEEE TPAMI}, \emph{IJCV}, CVPR, and ICCV. His research interests include big data, data mining, cloud computing, granular computing, and rough sets. He is a Fellow of IRSS and a Senior Member of ACM and IEEE.
    \end{IEEEbiography}

\endgroup
\par

\end{document}